\documentclass{article}

\usepackage[preprint]{corl_2026} 

\usepackage{algorithm}
\usepackage{algpseudocode}
\usepackage{amsmath}
\usepackage{amssymb}
\usepackage{amsthm}
\usepackage{bm}
\usepackage{cancel}
\usepackage{mathtools}
\usepackage{multicol}
\usepackage[most]{tcolorbox}
\usepackage{graphicx}
\usepackage{subcaption}
\usepackage{array}
\usepackage{wrapfig}
\usepackage{booktabs} 
\usepackage[framemethod=default]{mdframed}
\usepackage{xcolor}
\usepackage[most]{tcolorbox}
\usepackage[textsize=tiny]{todonotes}
\definecolor{darkpurple}{RGB}{72,36,117}
\newtcolorbox{intuitionbox}{
    enhanced, 
    colback=white,
    colframe=darkpurple, 
    boxrule=0.8pt,
    arc=2mm,
    left=2mm,
    right=2mm,
    top=5mm, 
    bottom=2mm,
    title={Geometric Intuition},
    coltitle=white,
    fonttitle=\sffamily\bfseries,
    attach boxed title to top left={
        xshift=4mm,
        yshift=-\tcboxedtitleheight/2 
    },
    boxed title style={
        colback=darkpurple,
        colframe=darkpurple,
        boxrule=0pt,
        arc=1mm, 
        left=2mm,
        right=2mm,
        top=1mm,
        bottom=1mm
    }
}

\theoremstyle{plain}
\newtheorem{theorem}{Theorem}[section]

\newtheorem{lemma}[theorem]{Lemma}

\theoremstyle{definition}

\theoremstyle{remark}

\newcommand{\question}[1]{\textbf{Question:} #1}
\definecolor{softteal}{RGB}{70, 190, 180} 
\newcommand{\bt}[1]{\textcolor{softteal}
{\textbf{#1}}} 
\definecolor{softpurple}{RGB}{150, 110, 190}
\newcommand{\hc}[1]{\textcolor{softpurple}{\textbf{#1}}}

\definecolor{codegreen}{rgb}{0,0.6,0}
\definecolor{codegray}{rgb}{0.5,0.5,0.5}
\definecolor{codepurple}{rgb}{0.58,0,0.82}
\definecolor{backcolour}{rgb}{0.98,0.98,0.98}

\lstdefinestyle{mystyle}{
    backgroundcolor=\color{backcolour},   
    commentstyle=\color{codegreen},
    keywordstyle=\color{magenta},
    numberstyle=\tiny\color{codegray},
    stringstyle=\color{codepurple},
    basicstyle=\ttfamily\footnotesize,
    breakatwhitespace=false,         
    breaklines=true,                 
    captionpos=b,                    
    keepspaces=true,                 
    numbers=left,                    
    numbersep=5pt,                  
    showspaces=false,                
    showstringspaces=false,
    showtabs=false,                  
    tabsize=4
}
\def\eqref#1{equation~\ref{#1}}

\def\1{\bm{1}}

\def\va{{\bm{a}}}

\def\vg{{\bm{g}}}

\def\vs{{\bm{s}}}

\def\mG{{\bm{G}}}

\DeclareMathAlphabet{\mathsfit}{\encodingdefault}{\sfdefault}{m}{sl}
\SetMathAlphabet{\mathsfit}{bold}{\encodingdefault}{\sfdefault}{bx}{n}

\def\gA{{\mathcal{A}}}

\def\gG{{\mathcal{G}}}

\def\gJ{{\mathcal{J}}}

\def\gM{{\mathcal{M}}}

\def\gP{{\mathcal{P}}}

\def\gS{{\mathcal{S}}}

\def\gX{{\mathcal{X}}}

\def\sR{{\mathbb{R}}}

\newcommand{\E}{\mathbb{E}}

\usepackage{makecell}

\DeclareMathOperator*{\argmax}{arg\,max}

\title{
    Mollified Value Learning
}

\author{
    Hrishikesh~Viswanath~$^{1}$\quad
    \textbf{Juanwu~Lu}~$^{2}$\quad
    \textbf{S.~Talha~Bukhari~$^{1}$}\quad \\
    \textbf{Mihir Chauhan~$^{1}$}\quad
    \textbf{Damon~Conover~$^{3}$}\quad
    \textbf{Ziran~Wang~$^{2}$}\quad
    \textbf{Aniket Bera~$^{1}$}
    \\
    $^{1}$~Department of Computer Science, Purdue University, USA \\
    $^{2}$~College of Engineering, Purdue University, USA \\
    $^{3}$~DEVCOM Army Research Laboratory, USA \\
    Correspondence to: \texttt{hviswan@purdue.edu} \\
}

\usepackage{hyperref}

\usepackage[capitalize,noabbrev]{cleveref}

\begin{document}

\maketitle


\begin{figure}[h!]
    \centering
    \begin{tabular}{ m{0.001\textwidth} m{0.206\textwidth} m{0.206\textwidth} m{0.206\textwidth} m{0.206\textwidth} }
        
        \centering\rotatebox{90}{\footnotesize \textbf{\texttt{DUAL+EIK}}} & 
        \includegraphics[width=1.105\linewidth]{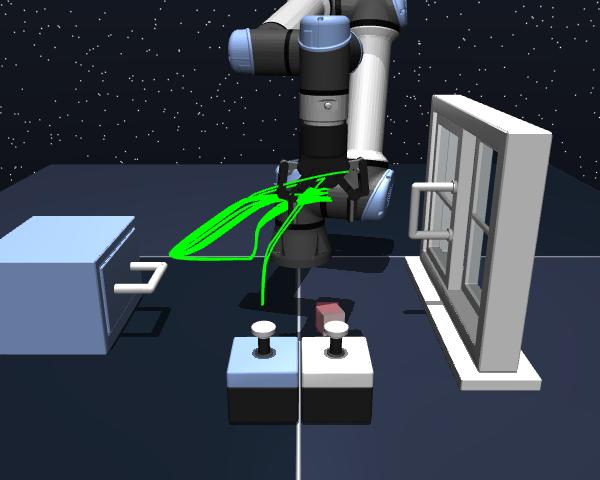} & 
        \includegraphics[width=1.105\linewidth]{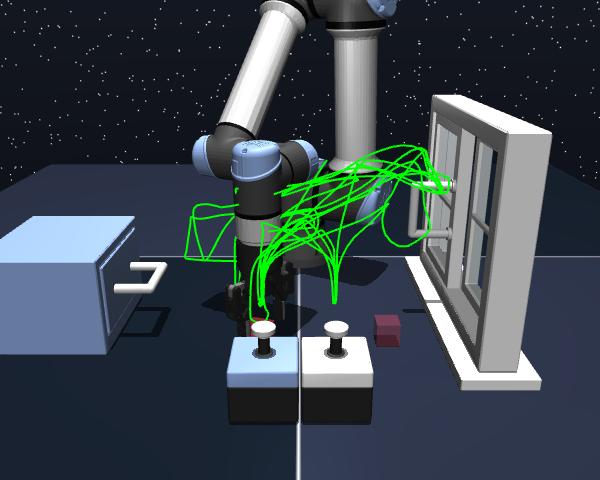} & 
        \includegraphics[width=1.105\linewidth]{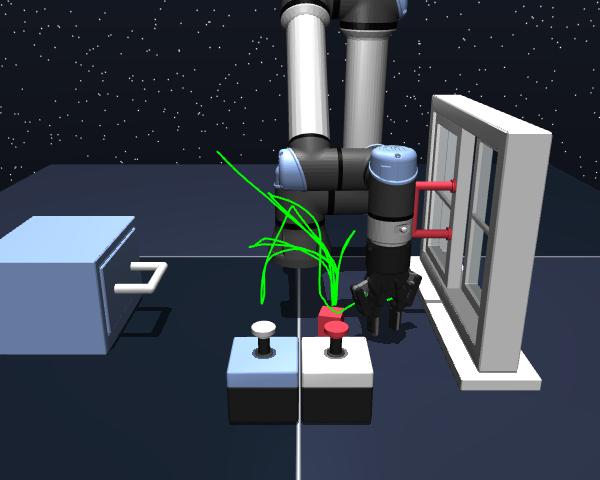} & 
        \includegraphics[width=1.105\linewidth]{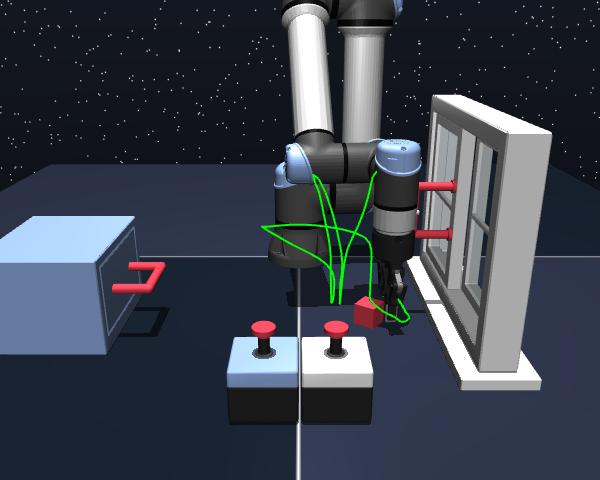} \\
        
        \centering\rotatebox{90}{\footnotesize \textbf{\texttt{OURS}}} & 
        \includegraphics[width=1.105\linewidth]{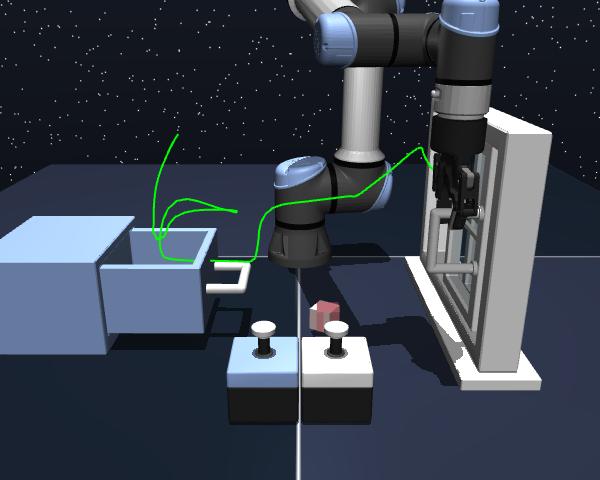} & 
        \includegraphics[width=1.105\linewidth]{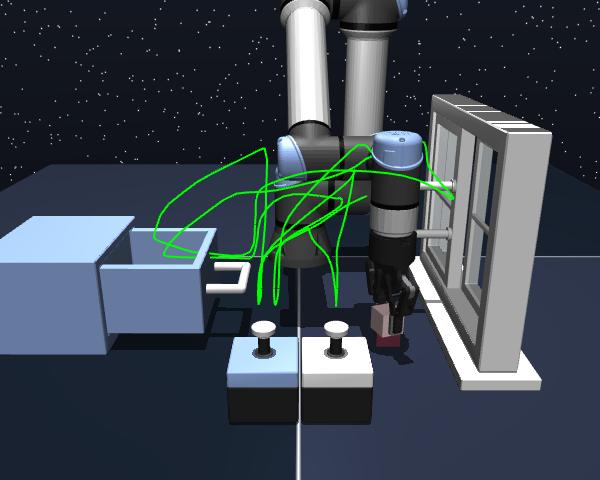} & 
        \includegraphics[width=1.105\linewidth]{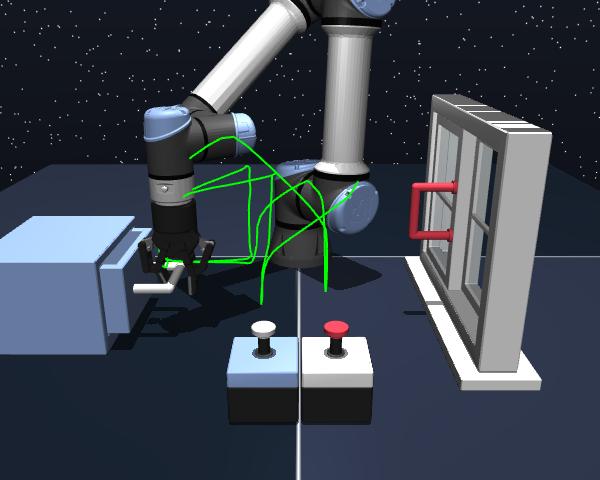} & 
        \includegraphics[width=1.105\linewidth]{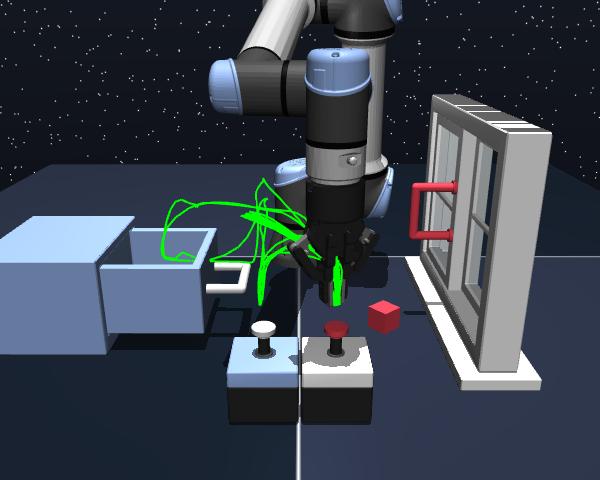} \\

        & 
        \centering\footnotesize (a) \textit{Open} & 
        \centering\footnotesize (b) \textit{Rearrange (Med)} & 
        \centering\footnotesize (c) \textit{Put in Drawer} & 
        \centering\footnotesize (d) \textit{Rearrange (Hard)} \\
        
    \end{tabular}

    \caption{\textbf{Contrasting Manipulation Tasks.} By imposing our proposed formulation on the value function for \textit{offline goal-conditioned reinforcement learning}, the policy successfully executes complex manipulation tasks where the Eikonal baseline \cite{giammarino2025physics} struggles. As shown (Bottom), our approach yields better trajectories across various manipulation tasks.}
\label{fig:qualitative_tasks}
\vspace{-1em}
\end{figure}
\begin{abstract}
    Offline goal-conditioned reinforcement learning (GCRL) learns goal-reaching behaviors from static datasets, but accurate value estimation remains challenging under limited state-action coverage. Existing physics-informed approaches address this by imposing pointwise distance-like geometric constraints derived from Hamilton--Jacobi--Bellman (HJB) optimality principles, often through first-order partial differential equations such as the Eikonal equation. However, enforcing local consistency through explicit differential structure can become unstable in complex, high-dimensional environments. Our key insight is to instead reinterpret distance-like constraints as an expectation over a local spatial measure. By aggregating constraints over this measure rather than evaluating them pointwise, the objective acts as a spatial mollifier, inducing distance-like value geometry without requiring expensive differential operators. We refer to this as Mollified Value Learning (MVL). Experiments across navigation and high-dimensional robotic manipulation tasks show that MVL learns structured, value representations, improving goal-reaching performance, when used with implicit value representation learning methods. Open-source codes are available at~\url{https://github.com/HrishikeshVish/MVL}.
\end{abstract}


\section{Introduction}
\label{sec: intro}

Offline goal-conditioned reinforcement learning (GCRL)~\cite{yang2022rethinking} trains agents to reach desired states using static, pre-collected datasets, enabling applications such as robotic manipulation and navigation. A central challenge in this setting is accurately estimating the goal-conditioned value function (GCVF)~\cite{yang2022rethinking}, which serves both as a critic for policy extraction and a guide for planning. While standard value parameterizations~\cite{park2023hiql, ma2022vip, schaul2015universal} are empirically effective, their updates rely primarily on temporal-difference recursion and therefore lack the structural inductive biases needed to capture the goal-reaching geometry. This often leads to uninformative or poorly structured value estimates, particularly in sparse-reward settings~\cite{park2023hiql, giammarino2025physics}.

To address this, recent works have explored structured geometric representations through physical priors, particularly in model-based and Koopman formulations~\cite{weissenbacher2022koopman, cheng2023look}, while the recent work \textit{dual-goal representations} explicitly models temporal distances between states within the value function itself~\cite{park2025dualgoalrepresentations}. Another promising direction is to impose geometric constraints derived from optimal control theory through partial differential equations (PDEs). In particular, recent physics-informed (PINN \cite{raissi2019physics}) approaches propose formulations based on Eikonal and Hamilton--Jacobi-Bellman (HJB) to model shortest-path geometry (i.e., a geodesic in continuous space), defining local distance-like cost-to-go structures in value functions~\cite{ni2025physics, ren2025physics, liu2025physics, ni2022ntfields, ni2024physics, giammarino2025physics, lien2024enhancing}, showing that such local geometry shaping alongside TD-learning improve goal-reaching performance. More broadly, stochastic forms~\cite{quer2024connecting} and second-order based methods~\cite{ni2023progressive, crandall1983viscosity} have been proposed to improve smoothness of the learned value functions. While these approaches provide principled formulations of optimality, achieving this through explicit differential structure remains challenging in high-dimensional offline RL due to sparse state-action coverage, numerical instability~\cite{munos1997convergent, munos1997reinforcement} in high-dimensional and complex kinematics, and the practical difficulties associated PINNs \cite{viswanath2026operator}. Meanwhile, the paradigm of weak-supervision, has emerged, among neural PDE solvers, as a tractable alternative to PINN learning over exact forms \cite{nam2024solving, li2023neural, lin2025enabling} where weak integral forms and stochastic relaxations of PDEs are instead used as training objectives. 

This raises a broader question: can the geometry of GCRL value functions be defined without relying on expensive differential formulations? We hypothesize that (1) reframing the shortest-path geometry  constraints as an expectation over a spatial measure captures the continuous variation in cost-to-go, effectively replacing pointwise derivatives with a tractable integral, and (2) defining this objective over a local state distribution induces a smoothing effect on the value geometry.

\textbf{Our Contribution.} Drawing from this, we propose Mollified Value Learning (MVL), a value learning framework that introduces shortest-path-like geometric inductive bias into offline goal-conditioned RL. Here, we recast the geometric consistency problem as solving for the expected cost-to-go constraint over local state distributions. We show that this induces a tractable value geometry, without requiring explicit higher-order differential operators. Empirically, we demonstrate that MVL learns structured, value representations and improves goal-reaching performance across challenging robotic navigation and manipulation tasks.

\section{Related Works}
\label{sec: related}

\noindent\textbf{Offline Goal-Conditioned Reinforcement Learning.}
Goal-conditioned Reinforcement Learning (GCRL)~\cite{kaelbling1993learning,schaul2015universal,levine2020offlinereinforcementlearningtutorial} is a class of RL that aims to learn goal-conditioned policies that take actions to achieve the goal state. Early works, such as hindsight goal relabeling, focus on adapting online algorithms to the offline settings~\cite{marcin2017hindsight,chebotar2021actionable,yang2022rethinking}, while consecutive works leverage techniques like discounting and advantage weighting~\cite{peters2010relative,peng2019advantage}. Recent offline approaches, including Implicit Q-Learning (IQL)~\cite{kostrikov2021offline} and Implicit Value Learning (GCIVL) \cite{park2025dualgoalrepresentations}, enable transitive planning via expectile regression without out-of-sample action queries, but often suffer from the curse of horizon due to compounding value errors over long trajectories. Hierarchical and horizon reduction methods~\cite{park2023hiql,park2025horizon, ahn2025option} mitigate this by decomposing tasks into sub-goals or using $n$-step value estimates. Meanwhile, quasimetric and contrastive learning approaches incorporate geometric properties in their architectures \cite{wang2023optimal, eysenbach2022contrastive}.
Most existing methods learn value functions from data without explicit physics priors, leading to biased estimates from bootstrapping and bad generalizability due to poor coverage of the state-action transition.

\noindent\textbf{Physics-Informed Value Learning and PDE mollifiers.}
Regularization of value estimation via physics-informed constraints has gained attention in planning and RL, with prior work enforcing first- or higher-order consistency in value gradients from a  control perspective~\cite{munos1997convergent,munos1997reinforcement}. Recent works have used Eikonal constraints in motion planning~\cite{ni2022ntfields, ni2024physics, ni2023progressive}. While some model the HJB directly~\cite{shilova2024learning, lien2024enhancing}, others use operator learning to combine neural operators with Eikonal PDEs for motion planning~\cite {matada2024generalizable}. Among RL methods include deriving offline RL objectives from the HJB equation~\cite{lien2024enhancing} and using Eikonal-based regularization to constrain the value gradients~\cite{giammarino2025physics}. However, such first-order constraints are often ill-posed for high-dimensional, complex problems, leading to numerical instability in neural approximations. Recent works on Hamilton-Jacobi equations over probability spaces and Wasserstein geometry \cite{gangbo2021finite, gangbo2008hamilton, badreddine2022solutions}, introduce weak notions of solutions for non-smooth optimal control problems. For deterministic offline GCRL problems, we instead adopt a simpler geometric perspective grounded in local state-space structure. A central tool in this setting is the mollifier, a kernel integral, yielding smoothing approximations that preserve large-scale structure while attenuating small-scale irregularities. Motivated by classical constructions in PDE theory \cite{bucur2016nonlocal, beghin2021nonlocal, evans2022partial} and recent control formulations \cite{du2025optimal},  we use mollification as a principled mechanism to define local geometric constraints of the value function.
\section{Preliminaries}
\label{sec: background}

In this section, we briefly introduce GCRL, and optimal control. And in \cref{sec: method}, we discuss our proposed method. Readers can find a table of notations in Appendix~\ref{appx: notation} for references.

\noindent\textbf{Goal-conditioned Reinforcement Learning.}
Let $\gP(\gX)$ be a set of valid probability distributions defined on the space $\gX$. We denote a goal-conditioned Markov decision process (GC-MDP) by $\gM=\left(\gS,\gA,p,r\right)$, with state space $\gS\subset\sR^{n}$, action space $\gA\subset\sR^{m}$, system transition dynamics given by $p(\vs_{t+1}\mid{\vs_{t}, \va}):\gS\times\gA\to\gP(\gS)$, and goal-conditioned reward function $r(\vs,\vg):\gS\times\gS\to\sR$, where the goal is defined in the state space $\vg\in\gS$. A goal-conditioned policy is denoted as $\pi(\va\mid{\vs,\vg}):\gS\times\gS\to\gP(\gA)$. GCRL aims to solve for an optimal policy $\pi^\ast$ that maximizes the expected cumulative discounted rewards to reach the goal state: 
    $\pi^{\ast}(\va\mid{\vs,\vg})\triangleq\argmax_{\pi\in\Pi}\E_{\bm{\tau}\sim{p^{\pi}(\bm{\tau})}}\left[\sum\limits_{t=0}^{T}\gamma^{t}\cdot{r(\vs_{t},\vg)}\right]$, 
where $\Pi$ is policy space, $\bm{\tau}\triangleq\left(\vs_{0},\va_{0},\ldots,\vs_{T-1},\va_{T-1},\vs_{T}\right)$ represents a trajectory, $p^{\pi}(\bm{\tau})$ is the distribution of such trajectories by executing policy $\pi$, and $T$ is the decision horizon. Finally, we denote a goal-conditioned state value function (GCVF) induced by policy $\pi$ as: $V^{\pi}(\vs,\vg)\triangleq\E_{\bm{\tau}\sim{p^{\pi}(\bm{\tau}\mid{\vs_{0}=\vs,\vg)}}}
\left[\sum\limits_{t=0}^{T}\gamma^{t}\cdot{r(\vs_{t},\va_t, \vg)}\right]$.

\noindent\textbf{GCRL with Optimal Control.}
\label{sec:optimal_control}
We study a continuous-time  dynamical system
\cite{park2021neural,pereira2019learning,todorov2009compositionality}
defined on a state space $\gS \subset \mathbb{R}^d$, $d\vs_t =
    f(\vs_t, \va_t)\,dt$
where $f(\cdot, \va)$ denotes the system dynamics,
$\va_t\in\gA\subset\mathbb{R}^d$ is the control input. We consider the finite-horizon optimal control problem with
objective functional $\gJ(\bm{\tau})
=
\mathbb{E}\!\left[
\int_0^T c(\vs_t,\va_t)\,dt
+
q_f(\vs_T,\vg)
\right]$,
where $c$ is the instantaneous cost and
$q_f:\gS\times\gS\to\mathbb{R}$ is the terminal cost. In the reinforcement learning setting, the value function is defined as the negative cost-to-go $V^*(\vs,\vg) = -J^*$, and for a sufficiently small $\Delta t$, it has been shown in recent works \cite{giammarino2025goal, giammarino2025physics}, that this can be expressed as $
V^*(\vs,\vg) = \inf_\va [V^*(s_{t + \Delta t}, \vg)-c(\vs_t, \va_t)\Delta t]$

\section{Mollified Value Learning}
\label{sec: method}

As motivated in \cref{sec: intro}, our main aim is to introduce a mollified value learning framework for offline goal-conditioned RL. In deterministic shortest-path problems, the optimal value function $V^*(\vs, \vg)$ naturally exhibits a quasimetric structure. This geometry provides a powerful inductive bias: the difference in value between any two states is fundamentally bounded by the spatial cost to transition between them. By extending this bounding principle across continuous neighborhoods, we formulate a mollified expectation over local state distributions, yielding a region-level geometric objective that can be approximated directly from sampled transitions.


\subsection{The Idea}
\label{subsec: integral}

Before discussing our approach, we first state the assumptions. 

\textbf{Assumptions} We assume the state space $\mathcal{S}$ and the action space $\mathcal{A}$ to be compact and assume the dynamics $f(\cdot, \va)$ to be deterministic and Lipschitz continuous in $\vs$. In the offline GCRL setting, we consider a continuous, bounded, nonnegative running cost $c: \gS\times\gG \rightarrow \mathbb{R}^d$, where $c$ denotes the cost of reaching a state $\vs \in \gS$ while pursuing $\vg \in \gG$. We assume this cost $c(\vs, \vg) = 1$ (i.e. unit running cost). The value function $V$, which we derive in the following section, is then intuitively defined by accumulating this cost along the trajectory. The optimal goal conditioned value function $V^*$ corresponds to the optimal-value of the shortest feasible path from the state $\vs$ to the goal $\vg$, i.e., a geodesic. 
\begin{wrapfigure}{r}{0.50\textwidth}
    \centering\vspace{-0pt}
    \includegraphics[width=0.50\textwidth]{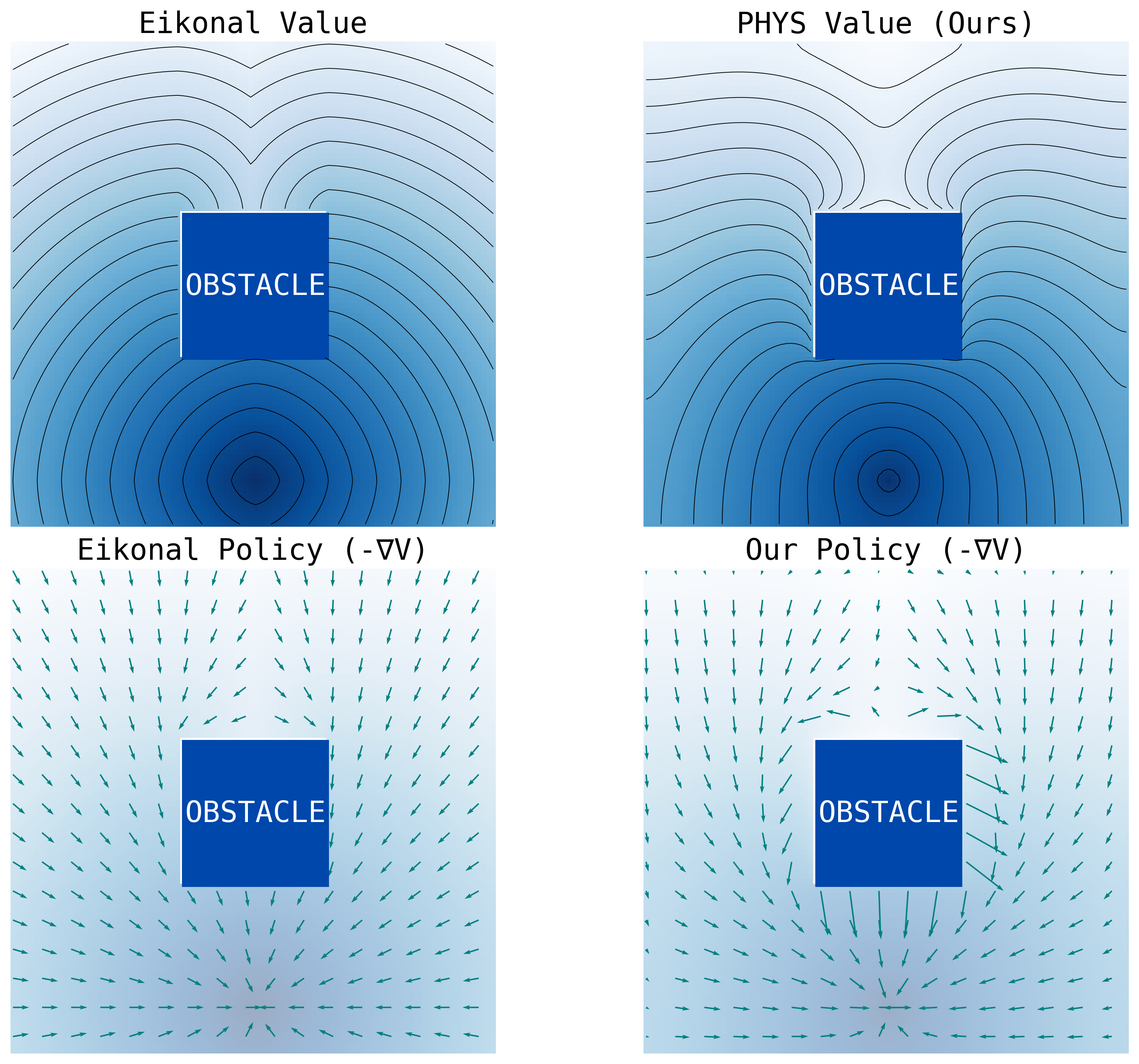}
    \caption{\textbf{Toy example with a single obstacle (blue square) in an arena, with the goal state at the bottom}: We solve the exact Eikonal form from \cite{giammarino2025physics} and our  integral form on a toy problem. }
    \label{fig:eik_hjb_toy}\vspace{-10pt}  
\end{wrapfigure}
Following these assumptions, we define the instantaneous cost $c(s,g) = 1$,  such that for a step defined in $\Delta t$, between $\vs$ and $\vs'$, the value $V^*(\vs, \vs')$ is given by $-V^*(\vs, \vs') \leq c(\vs, \vg)\Delta t + o(\Delta t)$. Over small displacements, we approximate $\Delta t \propto \|\vs' - \vs\|$. 

We state here that the optimal value function satisfies the triangle inequality \cite{wang2023optimal}, which is given by: 
$V(\vs, \vg) \geq V(\vs, \vs') + V(\vs', \vg)$. This follows the property that the optimal value function corresponds to the shortest feasible path between $\vs$ and $\vg$. The connections between this property and the HJB Hamiltonian in the continuum limit is discussed in \cref{subsec: solution}.  Following this, the local spatial restriction on the cost-to-go is then given by:
\begin{equation}
V(\vs', \vg) - [c(\vs,\vg)\Delta t + o(\Delta t)] \le V(\vs, \vg) 
\label{eq: spatial_constraint}
\end{equation}

We posit that shortest-path consistency can instead be relaxed as a spatial expectation principle, where cost-to-go relationships are defined in expectation over local state distributions.

We therefore consider a mollified relaxation of local distance-like consistency condition, defined over a local spatial measure. We highlight this on a simple toy problem in \cref{fig:eik_hjb_toy} which defines a simple obstacle and a goal state. We solve for the value field as an expectation over the local geometric constraints to show that this defines a similar gradient structure as the Eikonal based value field. 

For a given transition pair $(\vs, \vs')$ within $\Delta t$, we approximate the cost $c(\vs, \vg)\Delta t \approx c(\vs, \vg)\|s'- s\|$, consistent with our prior assumptions. Leveraging \cref{eq: spatial_constraint}, we define a function $\Phi(V,\vs,\vs',\vg):=V(\vs',\vg)-V(\vs,\vg)-c(\vs,\vg)\|\vs'-\vs\|$. The integral form is then given by \begin{equation}\mathcal{R}_\delta[V](\vs,\vg)
:=
\int
\Phi(V,\vs,\vs',\vg)
\rho_\delta(\vs'|\vs)\,d\vs'
\end{equation}

where, $\rho_\delta(\vs'|\vs)$ denotes a local transition kernel defined over $\|\vs'-\vs\|\leq \delta; \delta > 0$, such that $\int \rho_\delta(\vs' |\vs)ds'=1$, $\rho_\delta(\cdot | \vs)\geq 0$. Thus, $\Phi$ measures the local violation of distance constraints along sampled transitions, and $\mathcal{R}_\delta[V]$ enforces this consistency in expectation over the neighborhood of $\vs$. We detail the derivations and properties in Appendix \cref{appx: derivation}.

\paragraph{Practical Implementation}
For each state $\vs$, we sample a finite set of states $\{\vs'_1,\dots,\vs'_N\}$ such that $s'_i \sim \rho_\delta(\cdot|\vs)$ for $i=1,\dots,N$. In practice, for model-free offline settings, we let $\vs' \sim \mathcal{N}(\vs, \delta^2I)$ where, both $\delta$ and $N$ are treated as tunable hyperparameters (\cref{tab:ablation_hyperparams},\cref{tab:ablation_results_ND}). Under Gaussian probability measure, $
\mathcal{R}_\delta[V](\vs,\vg)
\approx
\frac{1}{N}
\sum_{i=1}^N
\Big[
V(\vs'_i,\vg)
-
V(\vs,\vg)
-
c(\vs,\vg)\|\vs'_i-\vs\|
\Big].
$

Using this, we define the geometric regularization as the empirical expectation: $\mathcal{L}_{\mathrm{MVL}}(\theta)
=
\mathbb{E}_{\vs,\vg}
\left[
\texttt{ReLU}\left(
\mathcal{R}_\delta[V_\theta](\vs,\vg)
\right)^2
\right]$, where the \texttt{ReLU} follows from the inequality in \cref{eq: spatial_constraint}.

\begin{intuitionbox}
Consider neighboring states within a small local region. Assuming locally uniform traversal velocity, the incremental transition cost scales proportionally with spatial displacement, i.e., $\Delta t \propto \|\vs-\vs'\|$. Under shortest-path consistency, the direct cost-to-go between a state and the goal should therefore not exceed the cost of first transitioning through a nearby state and then proceeding to the goal. Our objective measures violations of this local consistency relation. When the relation already holds, the penalty remains inactive. However, if transitions induce anomalously sharp local inconsistency, the integral form functions as a mollifier. 
\end{intuitionbox}

\section{Experiments}

\begin{table}[h!]
\centering
\caption{\textbf{RESULTS ON STATE-BASED TASKS WITH GCIVL.} Success rates across 4 seeds to highlight the impact of \texttt{MVL}. \bt{Soft teal} highlights methods within 5\% of the best performance, while \hc{soft purple} highlights an \texttt{MVL} variant when it strictly outperforms its corresponding base representation but falls outside the top 5\%.}
\label{tab:state_based_results}
\resizebox{\textwidth}{!}{%
\begin{tabular}{lccccccccc}
\toprule
\textbf{\texttt{ENVIRONMENT}} & \textbf{\texttt{ORIG}} & \textbf{\textbf{\texttt{ORIG}}-\textbf{MVL }} & \textbf{\texttt{VIB}} & \textbf{\textbf{\texttt{VIB}}-\textbf{MVL }} & \textbf{\texttt{VIP}} & \textbf{\texttt{TRA}} & \textbf{\texttt{BYOL}} & \textbf{\texttt{DUAL}} & \textbf{\textbf{\texttt{DUAL}}-\textbf{MVL }} \\
\midrule
\texttt{pointmaze-medium} & 78 $\pm$ 8 & 70 $\pm$ 2 & 69 $\pm$ 13 & \bt{92 $\pm$ 6} & 0 $\pm$ 1 & 3 $\pm$ 6 & 37 $\pm$ 7 & 76 $\pm$ 7 &   \bt{96 $\pm$ 3}\\
\texttt{pointmaze-large} & 52 $\pm$ 6 & 48 $\pm$ 4 & 50 $\pm$ 7 & \bt{80 $\pm$ 3} & 0 $\pm$ 0 & 1 $\pm$ 2 & 22 $\pm$ 12 & 46 $\pm$ 6 &   \bt{77 $\pm$ 1}\\
\texttt{antmaze-medium} & 71 $\pm$ 4 & 61 $\pm$ 5 & 68 $\pm$ 4 & \bt{75 $\pm$ 4} & 31 $\pm$ 5 & 22 $\pm$ 15 & 39 $\pm$ 5 & \bt{75 $\pm$ 4} & \bt{76 $\pm$ 5} \\
\texttt{antmaze-large} & 16 $\pm$ 3 & \hc{23 $\pm$ 2} & 9 $\pm$ 3 & \hc{26 $\pm$ 2} & 9 $\pm$ 2 & 22 $\pm$ 12 & 11 $\pm$ 5 & 28 $\pm$ 11 &  \bt{39 $\pm$ 3}\\
\texttt{humanoid-medium} & 27 $\pm$ 3 & 27 $\pm$ 3 & 24 $\pm$ 2 & \hc{30 $\pm$ 3} & 7 $\pm$ 3 & 21 $\pm$ 3 & 18 $\pm$ 5 & 29 $\pm$ 3 &   \bt{34 $\pm$ 3}\\
\texttt{humanoid-large} & 3 $\pm$ 0 & \hc{4 $\pm$ 1} & 3 $\pm$ 1 & \hc{4 $\pm$ 1} & 1 $\pm$ 0 & 2 $\pm$ 1 & 2 $\pm$ 1 & 3 $\pm$ 2 &   \bt{7 $\pm$ 2}\\
\texttt{antsoccer-arena} & \bt{47 $\pm$ 4} & 41 $\pm$ 1 & 34 $\pm$ 4 & 30 $\pm$ 3 & 2 $\pm$ 1 & 8 $\pm$ 2 & 11 $\pm$ 4 & 31 $\pm$ 3 &  27 $\pm$ 1\\
\midrule
\texttt{cube-single} & 52 $\pm$ 3 & \hc{68 $\pm$ 11} & \bt{90 $\pm$ 3} & 83 $\pm$ 4 & 40 $\pm$ 7 & 40 $\pm$ 5 & 51 $\pm$ 11 & 89 $\pm$ 3 & \bt{91 $\pm$ 1} \\
\texttt{cube-double} & 35 $\pm$ 5 & 30 $\pm$ 4 & 33 $\pm$ 3 & \bt{59 $\pm$ 3}  & 3 $\pm$ 2 & 7 $\pm$ 2 & 6 $\pm$ 4 & \bt{60 $\pm$ 4} & 52 $\pm$ 5 \\
\texttt{scene-play} & 46 $\pm$ 3 & \hc{52 $\pm$ 7} & 58 $\pm$ 1 & \bt{81 $\pm$ 7} & 23 $\pm$ 6 & 46 $\pm$ 6 & 44 $\pm$ 9 & 72 $\pm$ 6 &   \bt{84 $\pm$ 5}\\
\texttt{puzzle-3x3} & 5 $\pm$ 1 & \hc{8 $\pm$ 2} & \bt{14 $\pm$ 3} & 11 $\pm$ 2 & 3 $\pm$ 1 & 5 $\pm$ 1 & 0 $\pm$ 0 & 5 $\pm$ 1 &  \hc{7 $\pm$ 2}\\
\texttt{puzzle-4x4} & 14 $\pm$ 1 & \hc{25 $\pm$ 3} & 6 $\pm$ 3 & \hc{10 $\pm$ 4} & 1 $\pm$ 1 & 10 $\pm$ 3 & 1 $\pm$ 2 & 23 $\pm$ 3 &   \bt{34 $\pm$ 2}\\
\midrule
\texttt{D4RL/kitchen-partial}  & \bt{81 $\pm$ 9} & \bt{81 $\pm$ 11} & 75 $\pm$ 11 & \bt{85 $\pm$ 11} & 63 $\pm$ 0 & 72 $\pm$ 3 & 63 $\pm$ 0 & \bt{81 $\pm$ 6} & 78 $\pm$ 8 \\
\texttt{D4RL/kitchen-mixed}    & 77 $\pm$ 8 & \hc{79 $\pm$ 14} & \bt{81 $\pm$ 10} & \bt{84 $\pm$ 11} & 69 $\pm$ 0 & \bt{85 $\pm$ 3} & 69 $\pm$ 0 & 77 $\pm$ 8 & 76 $\pm$ 19 \\
\texttt{D4RL/kitchen-complete} & 54 $\pm$ 9 & \bt{59 $\pm$ 6} & \bt{56 $\pm$ 5} & 52 $\pm$ 9 & 25 $\pm$ 6 & 43 $\pm$ 7 & 0 $\pm$ 0  & 46 $\pm$ 14 & 46 $\pm$ 3 \\
\midrule
\textbf{\texttt{AVERAGE}} & 44 $\pm$ 4 & 45 $\pm$ 5 & 45 $\pm$ 5 & \bt{53 $\pm$ 5} & 18 $\pm$ 2 & 26 $\pm$ 5 & 25 $\pm$ 4 & 49 $\pm$ 5 &  \bt{55 $\pm$ 4}\\
\bottomrule
\end{tabular}%
}
\end{table}

\begin{figure}[!h] 
    \centering
    \setlength{\tabcolsep}{1pt} 
    
    \begin{subfigure}[b]{0.24\textwidth}
        \centering
        \includegraphics[width=\linewidth]{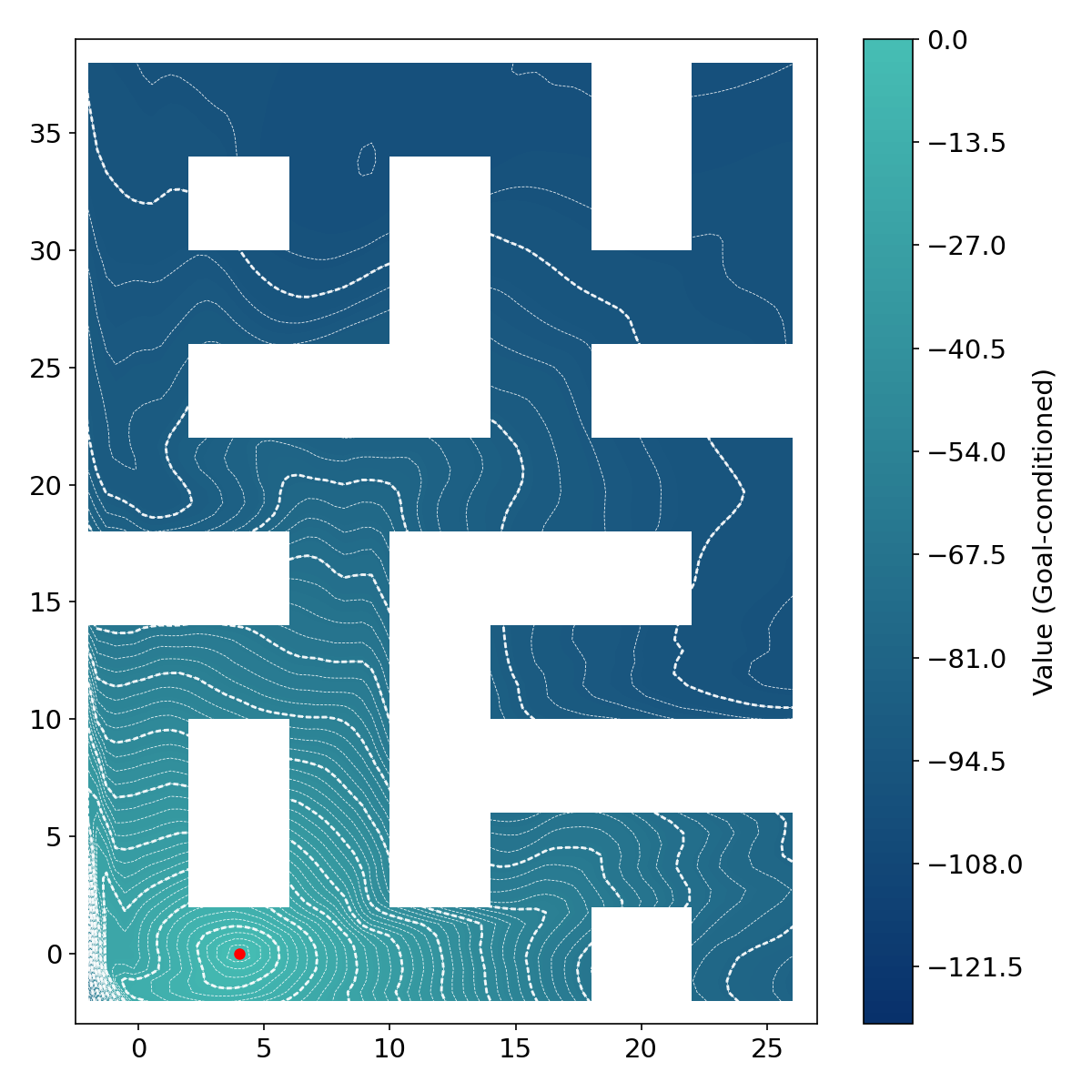}
        \caption{Orig: Baseline}
        \label{fig:orig_base}
    \end{subfigure}
    \hfill
    \begin{subfigure}[b]{0.24\textwidth}
        \centering
        \includegraphics[width=\linewidth]{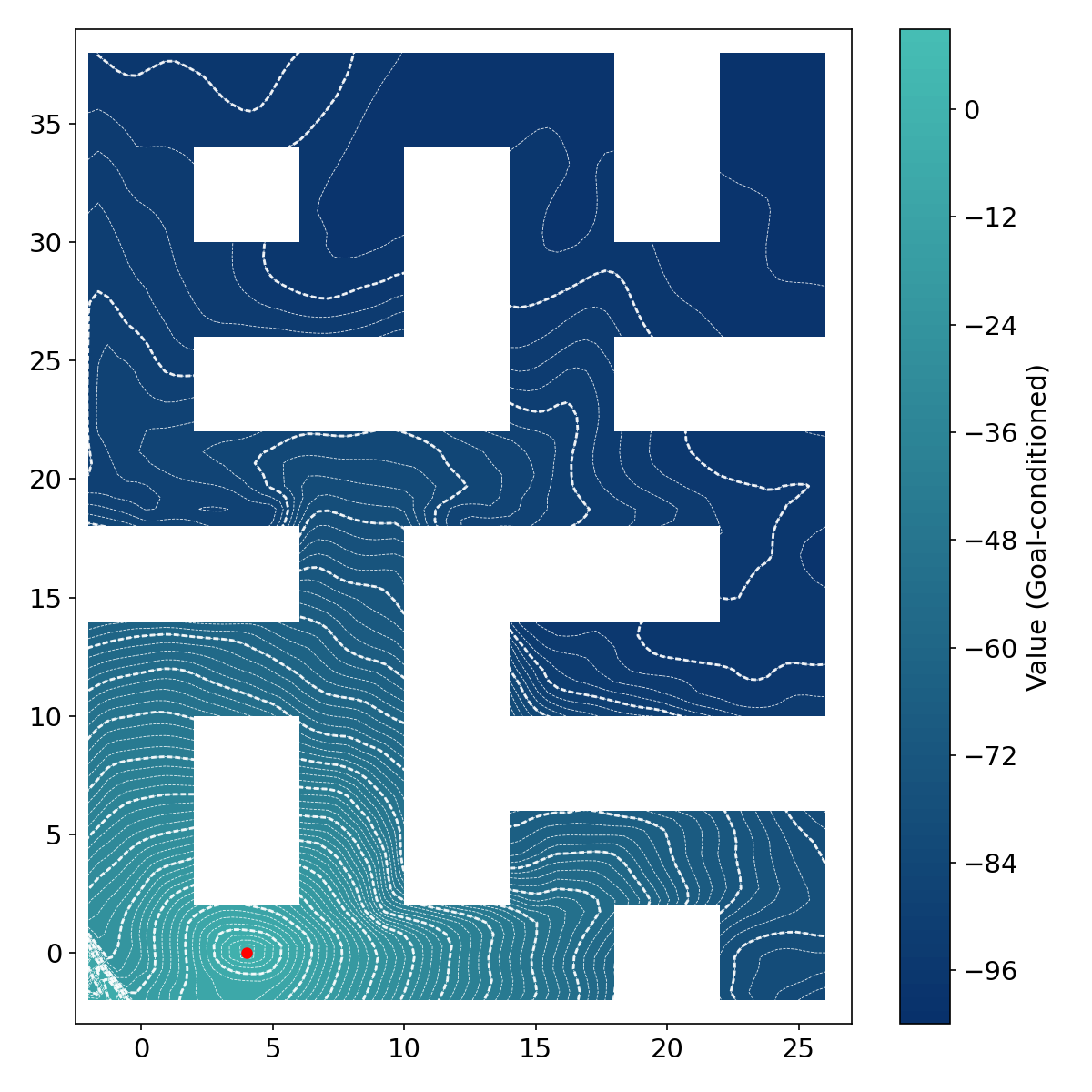}
        \caption{Dual: Baseline}
        \label{fig:dual_base}
    \end{subfigure}
    \hfill
        \begin{subfigure}[b]{0.24\textwidth}
        \centering
        \includegraphics[width=\linewidth]{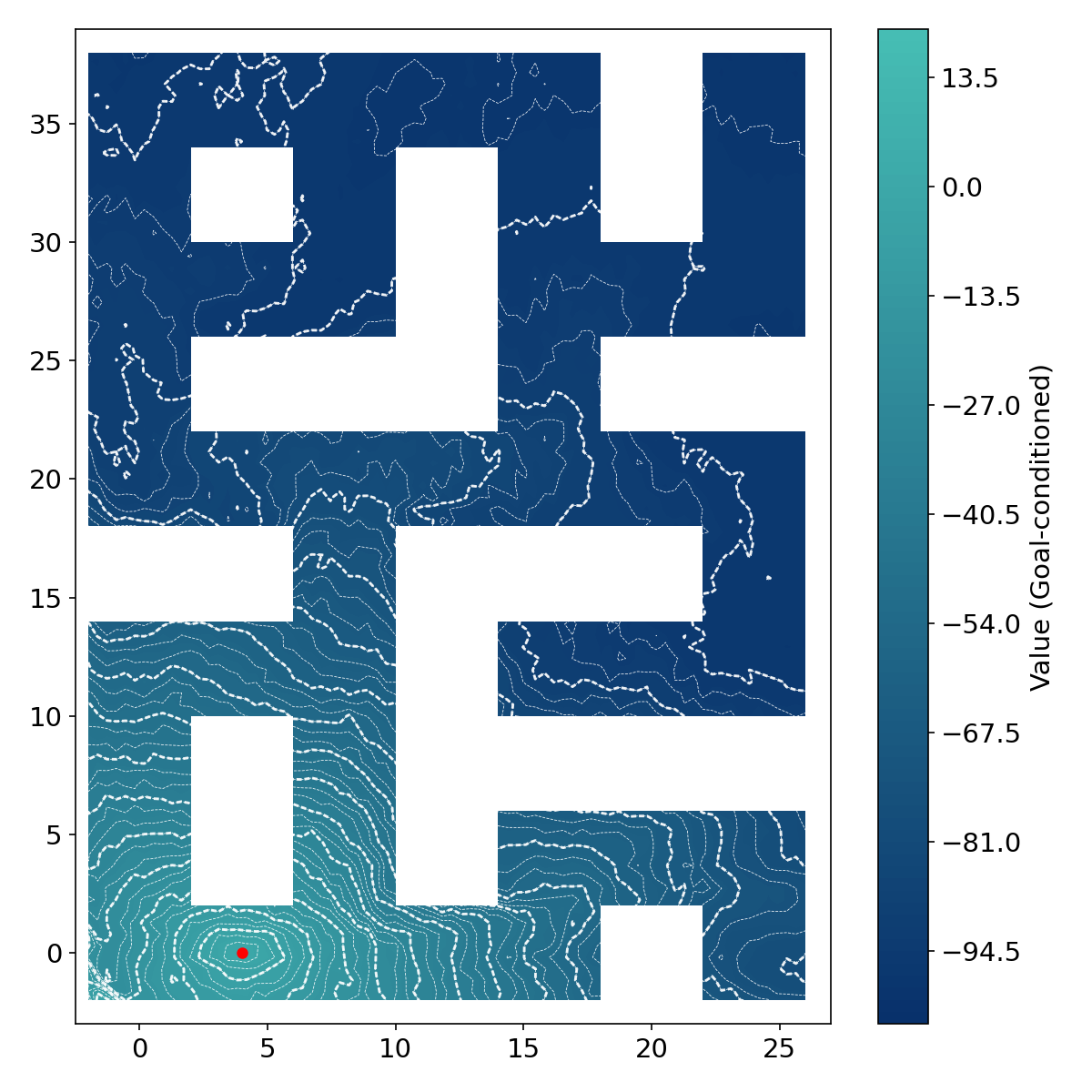}
        \caption{VIB: Baseline}
        \label{fig:vib_base}
    \end{subfigure}
    \hfill
    \begin{subfigure}[b]{0.24\textwidth}
        \centering
        \includegraphics[width=\linewidth]{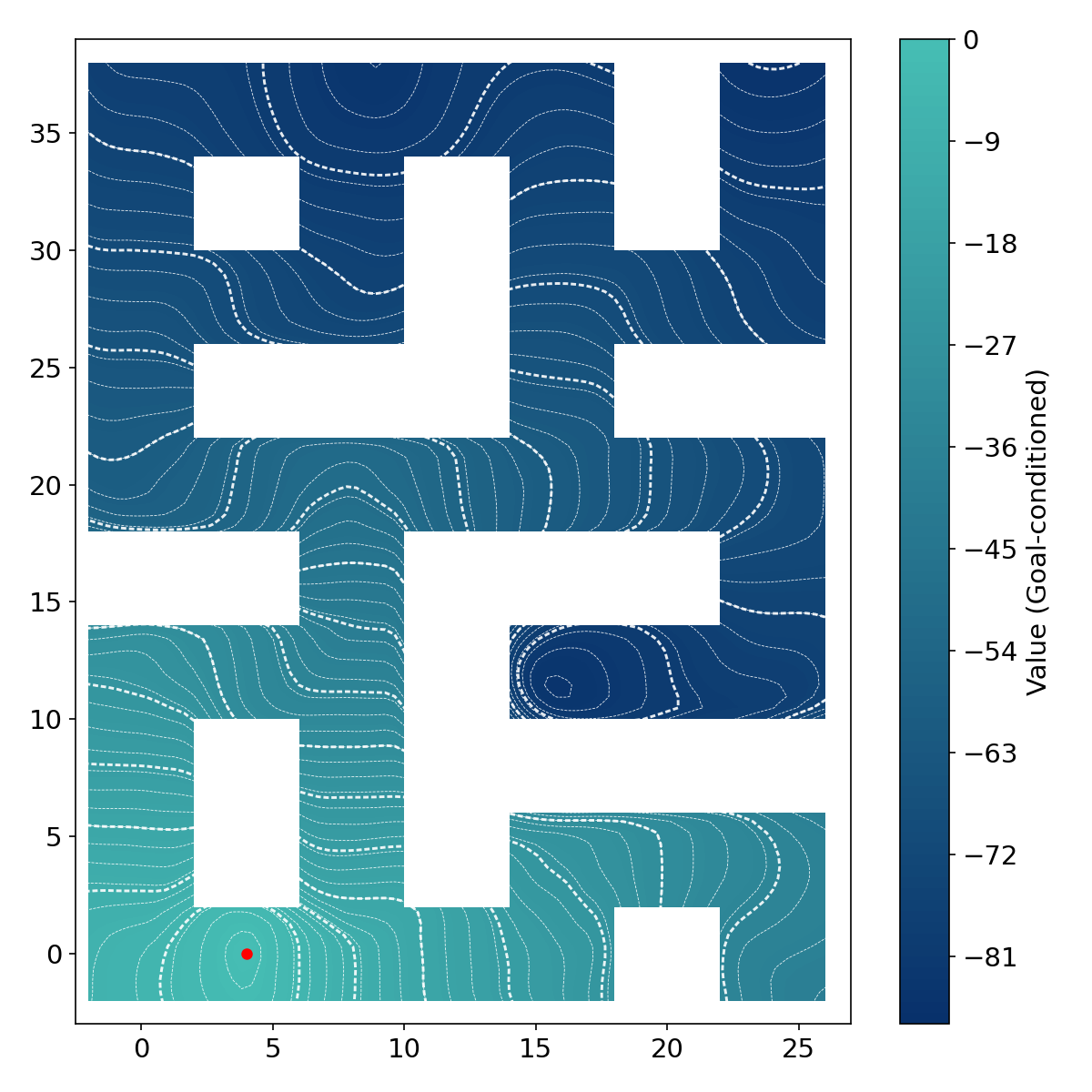}
        \caption{Dual: +Eikonal}
        \label{fig:dual_eik}
    \end{subfigure}

    \vspace{0.5em} 

    \begin{subfigure}[b]{0.24\textwidth}
        \centering
        \includegraphics[width=\linewidth]{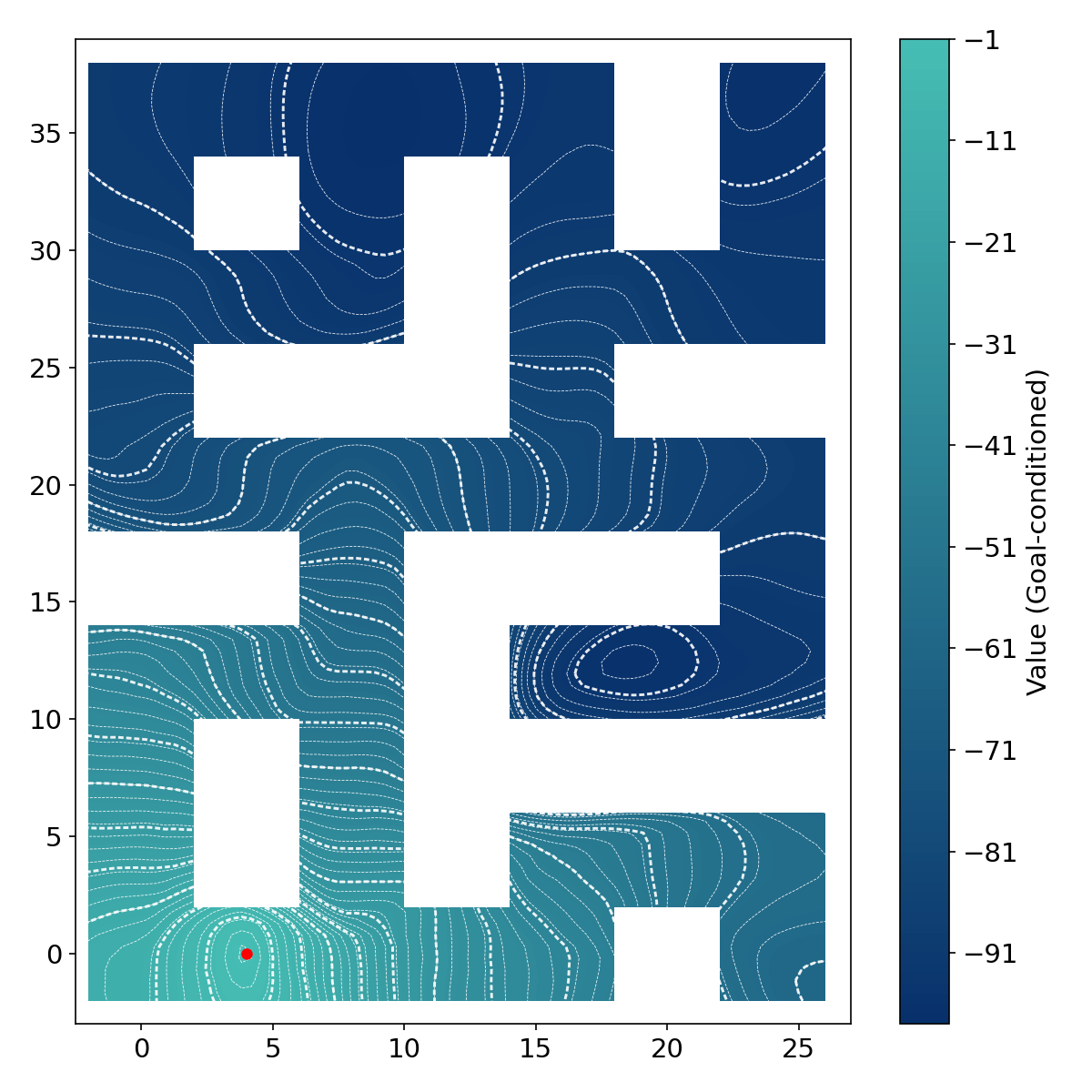}
        \caption{Orig: +Ours}
        \label{fig:orig_ours}
    \end{subfigure}
    \hfill
    \begin{subfigure}[b]{0.24\textwidth}
        \centering
        \includegraphics[width=\linewidth]{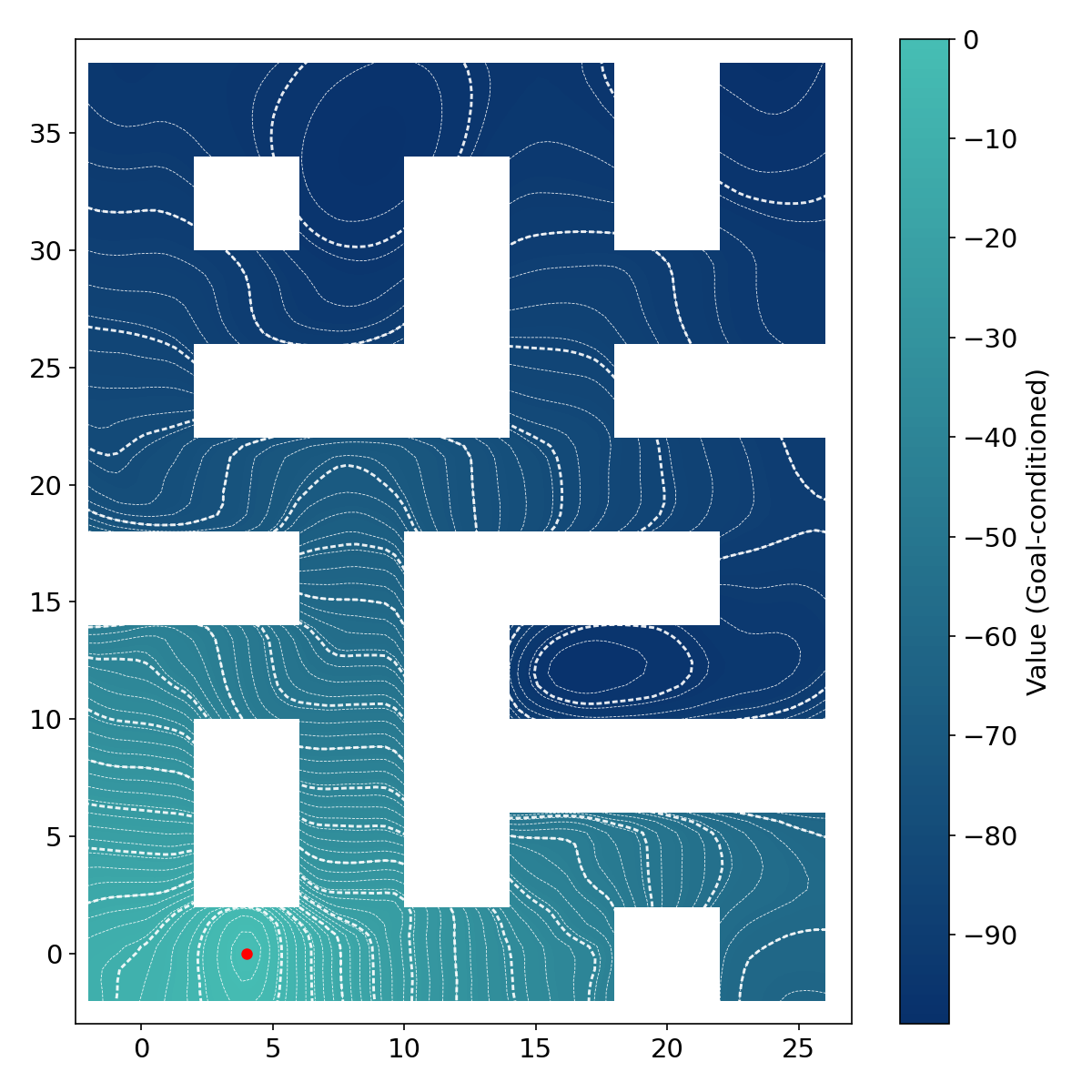}
        \caption{Dual: +Ours}
        \label{fig:dual_ours}
    \end{subfigure}
    \hfill
    \begin{subfigure}[b]{0.24\textwidth}
        \centering
        \includegraphics[width=\linewidth]{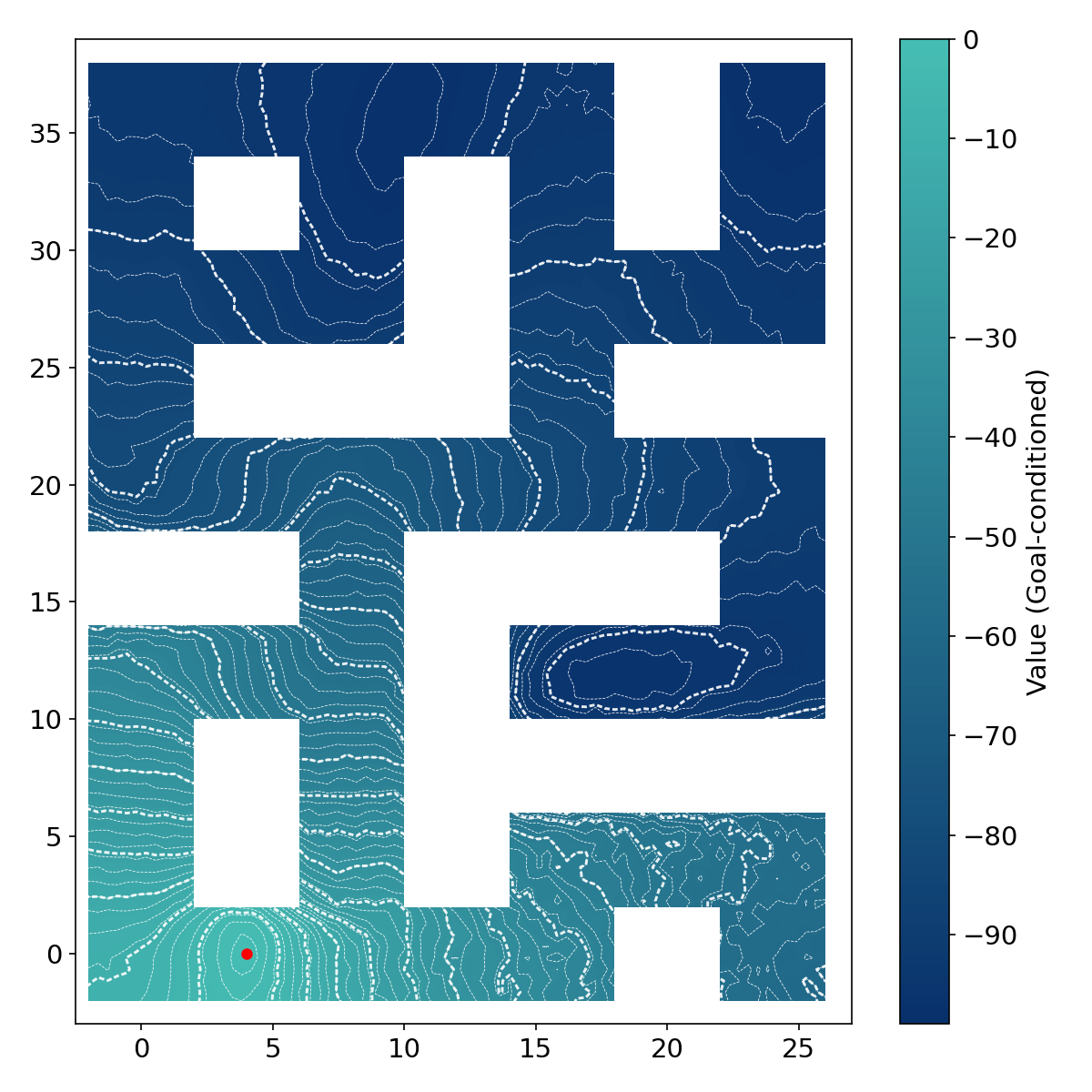}
        \caption{VIB: +Ours}
        \label{fig:vib_ours}
    \end{subfigure}
    \hfill
    \begin{subfigure}[b]{0.24\textwidth}
        \centering
        \includegraphics[width=\linewidth]{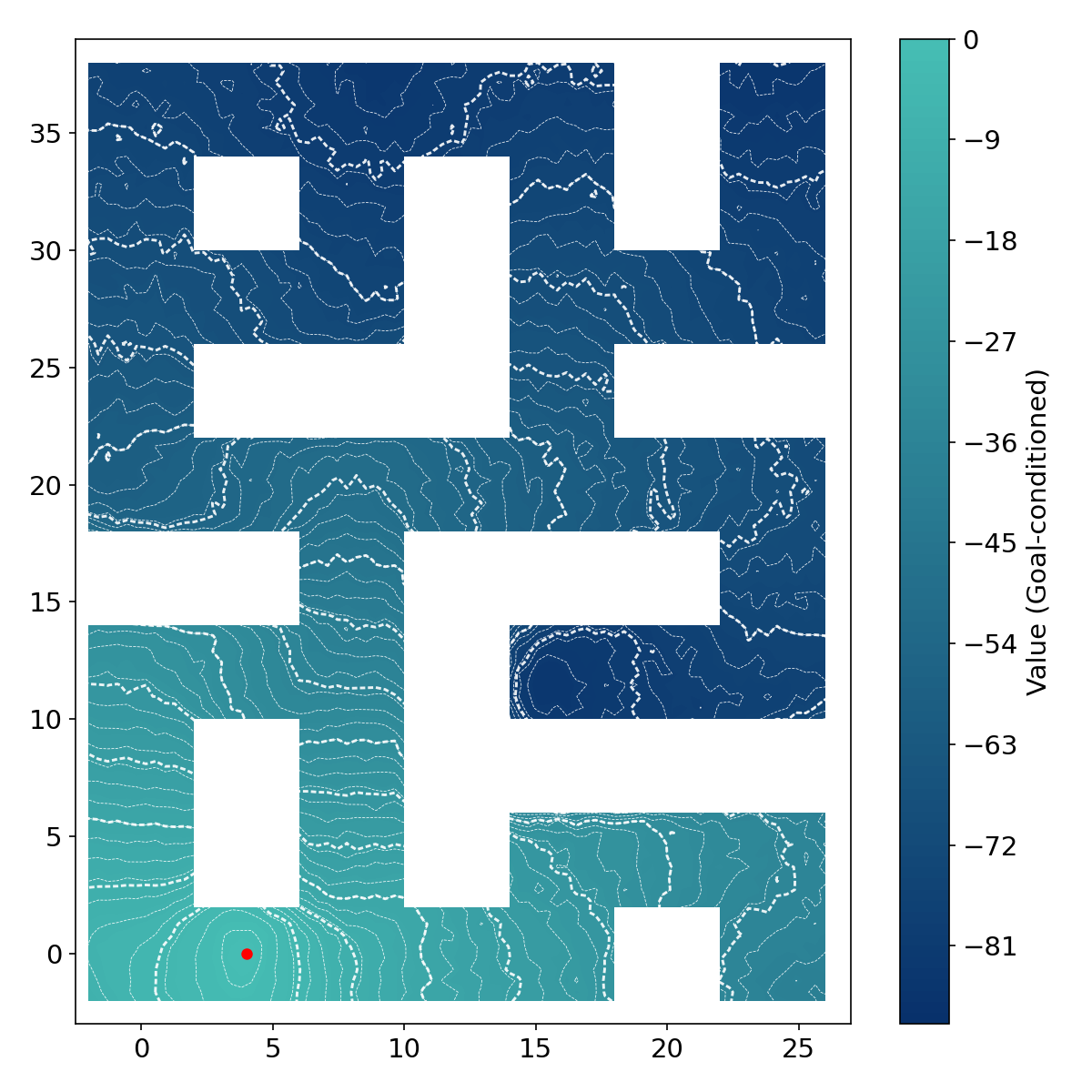}
        \caption{VIB: +Eikonal}
        \label{fig:vib_eik}
    \end{subfigure}

   \caption{%
        \textbf{Qualitative value contour ablation on PointMaze-Large.} 
        \textbf{Col 1 (Original):} GCIVL Baseline (\subref{fig:orig_base}) contours aren't aligned towards goal (bottom-left corner); ours (\subref{fig:orig_ours}) with aligned contours. 
        \textbf{Col 2 (Dual):} Baseline (\subref{fig:vib_base}) exhibits severe jitter; ours (\subref{fig:vib_ours}) with improved geometry. 
        \textbf{Col 3 (VIB):} Baseline (\subref{fig:dual_base}); ours (\subref{fig:dual_ours}) enforces value geometry. 
        \textbf{Col 4 (Eikonal):} Eikonal constraints (\subref{fig:vib_eik}, \subref{fig:dual_eik})}
    \label{fig:horizontal_ablation}
    \vspace{-1em}
\end{figure}

Following the experimental procedures of \cite{park2025dualgoalrepresentations}, we evaluate our method on the state-based offline goal-conditioned RL benchmark \texttt{OGBench} \cite{park2024ogbench} and the \texttt{Franka} Kitchen testbench from \texttt{D4RL}. We additionally do tests on a 7-dof Franka Panda on 2 custom designed tasks.

\noindent\textbf{Tasks and Datasets.}
We consider 13 state-based tasks from the \texttt{OGBench} datasets that cover both manipulation and navigation behaviors. The navigation tasks (\texttt{pointmaze}, \texttt{antmaze}, \texttt{humanoidmaze}) focus on navigating a robot through mazes of varying size, each with increasing complexity, with humanoid being the most complex (21 Degrees-of-Freedom in movement). In the manipulation experiments, we consider \texttt{cube, scene, and puzzle}, which involve a robotic arm manipulating, rearranging, or solving combinatorial puzzles.

\noindent\textbf{Baselines.}
Our policy algorithm extends the \texttt{GCIVL} algorithm presented in \cite{park2025dualgoalrepresentations}. We consider 5 representation strategies: Original (i.e., GCIVL), \texttt{VIB} \cite{shah2021rapid}, which trains representations through variational information bottleneck, \texttt{VIP} \cite{ma2022vip}, a metric-based representation, \texttt{TRA} \cite{myers2025temporal}, a contrastive representation, and \texttt{BYOL} \cite{lawson2025self}, a self-supervised approach. For the physics-informed strategies, we consider the Eikonal \texttt{EIK} from \cite{giammarino2025physics}. We compare against the quasimetric learning baselines \texttt{CRL} \cite{eysenbach2022contrastive}, \texttt{QRL} \cite{wang2023optimal}, \texttt{TDP} \cite{jurgenson2020sub}, \texttt{COE} \cite{pikekos2023efficient} and \texttt{TRL} \cite{park2025transitive} in \cref{tab:new_dataset_results} on \texttt{oraclerep} setting defined over a barebones representation with minimal information required to fulfil success criteria.

\subsection{Results} 
We present the results on the 13 \texttt{OGBench} tasks in \cref{tab:state_based_results}. We will discuss these results in detail in the following subsections. We will focus on three key questions that will drive our discussion. 

\begin{wraptable}{R}{0.55\textwidth}
\vspace{-4pt} 
\centering
\caption{\textbf{ABLATION EXPERIMENTS.}  \texttt{Eik-HIQL} applies Eikonal constraint to HIQL policy, \texttt{DUAL+EIK} enforces the first-order Eikonal constraint ($|\nabla V| \approx 1$).}
\label{tab:ablation_all_envs}
\resizebox{\linewidth}{!}{%
\begin{tabular}{lcccc}
\toprule
\textbf{\texttt{ENVIRONMENT}} & \textbf{\texttt{EIK-HIQL}} & \textbf{\texttt{HIQL+MVL}} & \textbf{\texttt{DUAL+EIK}} & \textbf{\texttt{DUAL+MVL}} \\
\midrule
\multicolumn{5}{l}{\textit{\textbf{PointMaze}}} \\
\texttt{point-nav-medium}      & \bt{93 $\pm$ 5} & \bt{95 $\pm$ 3} & \bt{92 $\pm$ 1} & \bt{96 $\pm$ 3} \\
\texttt{point-nav-large}       & 83 $\pm$ 9 & \bt{89 $\pm$ 5} & 79 $\pm$ 4 & 77 $\pm$ 1 \\
\texttt{point-nav-giant}       & 79 $\pm$ 13 & \bt{91 $\pm$ 4} & 73 $\pm$ 3 & 57 $\pm$ 7 \\
\texttt{point-nav-teleport}    & \bt{47 $\pm$ 10} & \bt{45 $\pm$ 10} & 44 $\pm$ 2 & \bt{50 $\pm$ 3} \\
\texttt{point-stitch-medium}   & \bt{96 $\pm$ 3} & \bt{95 $\pm$ 2} & \bt{93 $\pm$ 3} & \bt{93 $\pm$ 2} \\

\texttt{point-stitch-teleport} & 43 $\pm$ 9 & 42 $\pm$ 3 & \bt{50 $\pm$ 2} & \bt{52 $\pm$ 2} \\
\midrule
\multicolumn{5}{l}{\textit{\textbf{AntMaze}}} \\
\texttt{ant-nav-medium}        & \bt{95 $\pm$ 1} & \bt{97 $\pm$ 2} & 77 $\pm$ 4 & 76 $\pm$ 5 \\
\texttt{ant-nav-large}         & 86 $\pm$ 2 & \bt{92 $\pm$ 2} & 46 $\pm$ 3 & 39 $\pm$ 3 \\
\texttt{ant-nav-giant}         & \bt{67 $\pm$ 5} & \bt{68 $\pm$ 3} & 5 $\pm$ 2  & 1 $\pm$ 0 \\
\texttt{ant-nav-teleport}      & \bt{52 $\pm$ 4} & \bt{50 $\pm$ 4} & 36 $\pm$ 1 & 39 $\pm$ 2 \\
\midrule
\multicolumn{5}{l}{\textit{\textbf{AntSoccer}}} \\
\texttt{soccer-nav-arena}    & 19 $\pm$ 2 & \bt{43 $\pm$ 3} & 7 $\pm$ 2 & 27 $\pm$ 1 \\
\texttt{soccer-nav-medium}   & \bt{3 $\pm$ 2} & \bt{6 $\pm$ 2} & \bt{2 $\pm$ 1} & \bt{3 $\pm$ 1} \\
\texttt{soccer-stitch-arena} & 2 $\pm$ 0 & \bt{13 $\pm$ 3} & 1 $\pm$ 1 & \bt{10 $\pm$ 3} \\
\texttt{soccer-stitch-medium} & \bt{1 $\pm$ 0} & \bt{4 $\pm$ 1} & \bt{0 $\pm$ 0} & \bt{0 $\pm$ 0} \\
\midrule
\multicolumn{5}{l}{\textit{\textbf{Manipulation}}} \\
\texttt{cube-single-play} & 25 $\pm$ 1 & 26 $\pm$ 3 & 1 $\pm$ 0 & \bt{91 $\pm$ 1} \\
\texttt{scene-play}       & 52 $\pm$ 7 & 56 $\pm$ 4 & 13 $\pm$ 2 & \bt{84 $\pm$ 5} \\
\bottomrule
\end{tabular}%
}
\vspace{-15pt} 
\end{wraptable}

\question{\textbf{How does the mollifier influence the representation bias of implicit value functions?}}

\noindent\textbf{Answer:} We observe that different representation learning strategies inject distinct biases that distort the underlying value geometry (\cref{fig:horizontal_ablation}). However, introducing geometric constraints helps mitigate these biases by realigning the value contours toward the goal. Empirically, this spatial realignment translates directly to performance: our method improves over plain \texttt{IVL} in 80\% of \texttt{oraclerep} environments and outperforms all other quasimetric baselines with vanilla \texttt{IVL} base in 55\% of cases (\cref{tab:new_dataset_results}). Furthermore, while hierarchical policies inherently construct better value structures, our non-differential regularization matches the explicit differential constraints of \texttt{Eik-HIQL} across all settings (\cref{tab:ablation_all_envs}). Crucially, in the absence of hierarchical abstraction, applying our geometry to the strongest base representation (dual-goal) overcomes the structural limitations in complex manipulation tasks, driving success rates in \texttt{cube-single-play} and \texttt{scene-play} to 91\% and 84\% (\cref{tab:state_based_results}).

\question{\textbf{How sensitive is the value geometry to unmodeled environment dynamics or dataset stochasticity, and why does \texttt{MVL} succeed?}}

\noindent\textbf{Answer:} Our philosophy is to view the shortest-path geometry purely as a spatial routing mechanism, indicating the best directional progress toward the goal (\cref{fig:advantage_comparison}), independent of the agent's low-level dynamics. This approach is naturally suited for offline settings where true environment dynamics are inaccessible. If the environment is stochastic and the transitions are noisy, our mollifier inherently decouples the spatial geometry from this low-level noise by evaluating the local spatial distribution rather than fitting to individual, artifact-heavy transitions. Thus, MVL acts as a geometric low-pass filter; it models the macroscopic routing signal while averaging out dataset stochasticity and unmodeled dynamics. This decoupling is empirically validated in highly stochastic settings (\cref{tab:combined_ablations}). Under severe transition noise, where baseline constraints often collapse to near-zero success, our spatial expectation preserves the underlying routing geometry, yielding robust recoveries such as $0 \rightarrow 99$ on \texttt{cube-single-noisy} and $0 \rightarrow 53$ on \texttt{scene-noisy}.

\begin{table}[h!]
    \centering
    \caption{\textbf{ROBUSTNESS AND SENSITIVITY ABLATIONS.} \textbf{(a)} Performance on stochastic environments. \textbf{(b)} Ablations (\texttt{pointmaze-large}) across varying sample sizes ($N$) and radius ($\delta$). Additional envs in \cref{tab:ablation_results_ND}}
    \label{tab:combined_ablations}
    \vspace{2mm} 
    
    \begin{subtable}[t]{0.45\linewidth}
        \centering
        \vspace{-12pt}
        \caption{Noisy environments}
        \label{tab:ablation_noisy}
        \resizebox{0.8\linewidth}{!}{%
        \begin{tabular}{lcc}
        \toprule
        \textbf{\texttt{ENVIRONMENT}} & \textbf{\texttt{DUAL+EIK}} & \textbf{\texttt{DUAL+MVL}} \\
        \midrule
        \multicolumn{3}{l}{\textit{\textbf{Noisy Manipulation}}} \\
        \texttt{scene-noisy}        & 0 $\pm$ 0  & \bt{53 $\pm$ 4} \\
        \texttt{cube-single-noisy}  & 0 $\pm$ 0  & \bt{99 $\pm$ 1} \\
        \texttt{cube-double-noisy}  & 1 $\pm$ 0   & \bt{30 $\pm$ 3} \\
        \midrule
        \multicolumn{3}{l}{\textit{\textbf{Noisy Puzzle}}} \\
        \texttt{puzzle-4x4-noisy}   & 0 $\pm$ 0  & \bt{24 $\pm$ 1} \\
        \texttt{puzzle-4x5-noisy}   & 0 $\pm$ 0 & \bt{18 $\pm$ 1} \\
        \bottomrule
        \end{tabular}%
        }
    \end{subtable}\hfill%
    \begin{subtable}[t]{0.51\linewidth}
        \centering
        \vspace{-12pt}
        \caption{Hyperparameters}
        \label{tab:ablation_hyperparams}
        \resizebox{0.8\linewidth}{!}{%
        \begin{tabular}{lc c lc}
        \toprule
        \multicolumn{2}{c}{\textit{\textbf{Num Samples ($N$)}}} & & \multicolumn{2}{c}{\textit{\textbf{Radius ($\delta$)}}} \\
        \cmidrule{1-2} \cmidrule{4-5}
        \textbf{\texttt{PARAM}} & \textbf{\texttt{SUCCESS}} & & \textbf{\texttt{PARAM}} & \textbf{\texttt{SUCCESS}} \\
        \midrule
        $N=1$  & 57 $\pm$ 3 & & $\delta=10^{-4}$ & 55 $\pm$ 2\\
        $N=5$  & \bt{77 $\pm$ 2} & & $\delta=10^{-3}$ & 57 $\pm$ 3\\
        $N=10$ & 76 $\pm$ 2 & & $\delta=10^{-2}$ & 77 $\pm$ 3\\
        $N=20$ & \bt{77 $\pm$ 2} & & $\delta=10^{-1}$ & \bt{79 $\pm$ 2} \\
        $N=25$ & \bt{77 $\pm$ 3} & & $\delta=10$ & 64 $\pm$ 4 \\
        \bottomrule
        \end{tabular}%
        }
    \end{subtable}
    \vspace{-0.5em}
\end{table}

\begin{figure}[h!]
    \centering
    \includegraphics[width=1\linewidth]{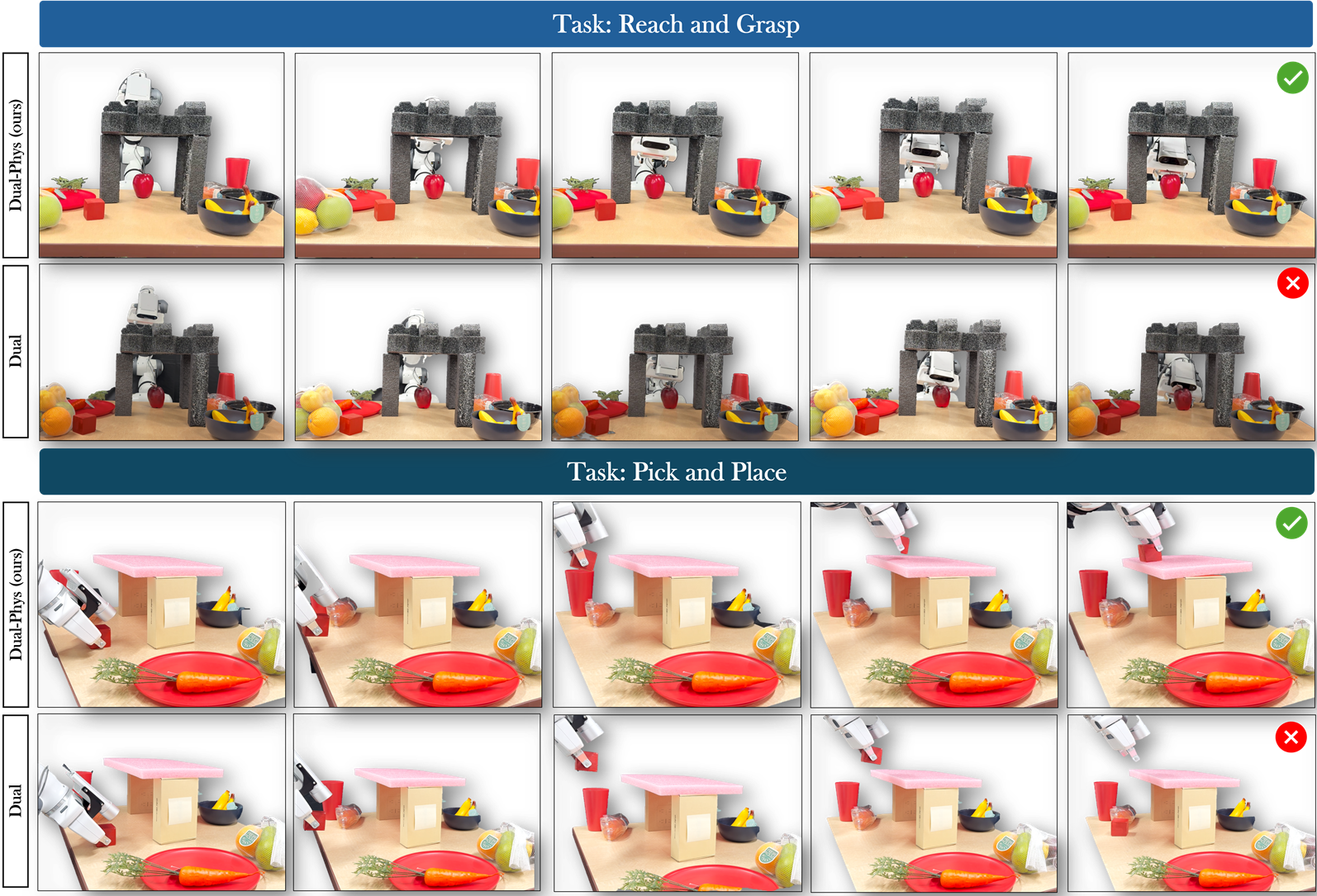}
    \caption{\textbf{Real-world demonstrations} \textbf{Top} In the Reach and Grasp task, without regularization, the arm fails to position itself to grab the apple and instead drags along the table. \textbf{Bottom} In Pick and Place, regularization helps position correctly before releasing the gripper. }
    \label{fig:realworldtasks}
\end{figure}
\begin{figure}[h!]
    \centering
    \begin{subfigure}[b]{0.48\textwidth}
        \centering
        \includegraphics[width=\linewidth]{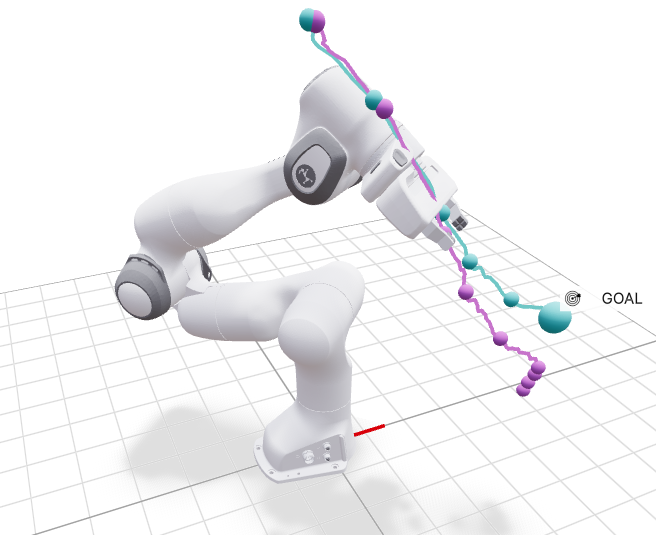}
        \caption{\texttt{Reach and Grasp}}
        \label{fig:realworld_reach}
    \end{subfigure}
    \hfill
    \begin{subfigure}[b]{0.48\textwidth}
        \centering
        \includegraphics[width=\linewidth]{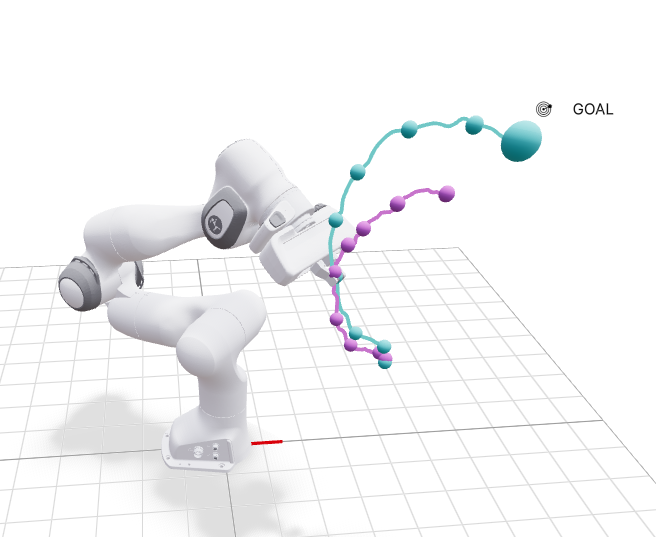}
        \caption{\texttt{Pick and Place}}
        \label{fig:realworld_pick}
    \end{subfigure}
    
    \caption{\textbf{MVL deployed on Franka Emika Panda.} We use \texttt{Viser} to visualize the execution trajectories of our policy, represented as soft teal waypoints (Purple represents base dual policy). This visualizes the routing of the end-effector toward the goal in \texttt{Reach and Grasp} (left) and the more complex \texttt{Pick and Place} (right) task, as shown in \cref{fig:realworldtasks}}
    \label{fig:franka_tasks}
\end{figure}

\question{\textbf{How does this approach scale to real-world settings?}}

\noindent\textbf{Answer:} In real-world physical systems, data is inherently prone to sensor variance, calibration errors, and unmodeled friction. Policies are thus required to be robust to these complex behaviors. By defining the value geometry as a spatial expectation, we investigate whether our formulation can naturally absorb this real-world noise. We validate this on a \texttt{Franka Emika Panda} across \texttt{Reach and Grasp} and \texttt{Pick and Place} tasks (Fig.~\cref{fig:realworldtasks}). For each task, we collect 100 demonstrations. Each trajectory is represented as 7-DoF joint states with binary gripper values, and actions are defined as joint-space deltas with gripper targets. The data is converted into an OGBench-compatible offline dataset where the goal states are inferred from sequential observations. Across both tasks, regularization significantly improves goal alignment. In \texttt{Reach and Grasp}, success improves from 2/10 without regularization to 8/10 with it. In \texttt{Pick and Place}, the unregularized policy consistently fails to reach goal proximity, while the regularized model achieves 10/10 success.

\noindent\textbf{Limitations}:
Our formulation defines local consistency through Euclidean distance and isotropic Gaussian kernels over the state space. Importantly, these kernels are not intended to model admissible system dynamics or feasible trajectories; rather, they define the spatial support of the mollification operator used to aggregate local quasimetric consistency constraints. While this avoids explicitly modeling nonholonomic constraints and contact dynamics, the relaxation allows us to simplify high dimensional state-spaces where the dynamics is unknown. We acknowledge, that this assumption leads us to an approximation of distance constraints rather than exact HJB Hamiltonian,  restricting us to a class of solutions that might not be optimal for all MDPs. Despite this simplification, the resulting objective remains effective across complex state-based environments. We additionally observe that in visual domains, such as Powderworld (\cref{tab:powderworld_results}), enforcing consistency in the \textit{Impala} latent space yields modest but gains (10\% vs 7\% on \texttt{powderworld-hard}). Since these representation spaces are not strictly Euclidean, the induced geometric assumptions may be insufficient. Overall, this approach belongs to a broader class of methods that replace exact optimality constraints with simplified geometric inductive biases. 
\section{Conclusion}
\label{sec:conclusion}

We introduced a geometric regularization framework for Goal-Conditioned RL based on localized consistency of value structure. The resulting formulation provides a simple and tractable inductive bias that improves value learning in high-dimensional manipulation tasks. A natural direction for future work is to extend these ideas to learned or manifold-aware representation spaces, particularly in pixel domains where Euclidean assumptions in latent space may not hold.


\clearpage

\bibliography{bibtex/juanwu}

\clearpage
\appendix
\onecolumn

\section{Derivations}
\label{appx: derivation}

\subsection{Assumptions}
\paragraph{Assumption 1 (State and Action Space)}: The state space $\mathcal{S} \subset \mathbb{R}^n$ and the action space $\mathcal{A}$ are compact and convex.
\paragraph{Assumption 2 (Dynamics)}: We consider the system dynamics $f(\cdot, a)$ to be Lipschitz continuous in $s$ (uniformly in $a$). We consider the relaxation such that the system follows unit-speed, isotropic dynamics where $f(s,a) = a$ and $\|a\| \le 1$. 
\paragraph{Assumption 3 (Action-Independent Cost)}: The cost $c(s,g)$ is continuous, bounded and nonnegative. Specifically, $c(g,g) = 0$, and $c(s,g) = 1$ for all $s \neq g$. The value function is obtained by accumulating this cost along the trajectory. 
\paragraph{Assumption 4 (Triangle Inequality)}: Let $d(x,y) : \mathbb{R}^d \times \mathbb{R}^d \rightarrow \mathbb{R}$ denote a distance like quasimetric such that $d(x,z) \leq d(x,y) + d(y, z), \forall x, y, z \in \mathbb{R}^d$, where $d(\cdot, \cdot)$ is locally Lipschitz such that $|d(x,z) - d(y,z)| \leq \|x-y\|$, and the value function satisfies $V^*(x, z) = -d(x,z)$
\subsection{Solution to Optimal Control}
\label{subsec: solution}


The first order approximation of this expression leads to the Hamiltonian, defined by the Hamilton-Jacobi-Bellman (HJB) equation, given by $H(\vs_t, \vg, \nabla V^*) = [f(\vs_t, \va_t)^\top \nabla V^*(\vs_t, \vg)-c(\vs_t, \va)]$.  

Given a system dynamics as specified by~\cref{sec:optimal_control}, the condition for the optimality of a control \cite{goldys2006second, nevistic1996constrained} is then characterized by 
\begin{equation}
   0 =
    \min_{\va}\left[f(\vs_t,\va_{t})^\top\nabla{V^{\ast}(\vs,\vg)} - c(\vs_{t},\va)
    \right]
    \label{eq: sec_ord_hjb}
\end{equation}
 Under an infinite sample budget and on-policy data, the Q-learning frameworks used in GCRL converge to the ideal value function. However, these are often violated in the offline setting, where datasets are biased. In such cases, the Hamiltonian in \eqref{eq: sec_ord_hjb} provides local geometric inductive bias to the learned value function, enabling it to model the goal-conditioned distance field. Solving~\eqref{eq: sec_ord_hjb} to find the minimizer of $\va$ yields an optimal policy. However, in practical robot learning problems, the value function may develop non-smooth regions (e.g., shocks or intersecting characteristic structures), where the gradients of the PDE are unstable or undefined in a classical sense \cite{crandall1983viscosity, fleming2006controlled, yong1999stochastic}.
\begin{lemma}
By the principle of optimality, the value function $V(s,g)$ is the unique solution to the static Hamilton-Jacobi-Bellman equation
$$0 = \inf_{a \in \mathcal{A}} \left[ \nabla_s V(s,g)^\top f(s,a) -c(\vs_{t},\va)\right]$$ 
The HJB equation encodes the relationship between the system dynamics, the cost function, and the value function, where the right-hand side: 
$$\mathcal{H}(s, g, \nabla_s V(s,g)) \equiv \inf_{a \in \mathcal{A}} \left[ \nabla_s V(s,g)^\top f(s,a) -c(\vs_{t},\va)\right]$$
is referred to as the Hamiltonian.
\end{lemma}

\begin{lemma} 
 For any valid transition $(s, s')$ generated by $a \in \mathcal{A}$ over time $\Delta t$ such that $s' = s + f(s,a)\Delta t$, by triangle inequality, the value function $V(s,g)$ satisfies, $V(s,g) \geq V(s, s') + V(s', g)$. In the one-step limit, $V(s,s') \leq -[c(s,g)\Delta t + o(\Delta t)]$. Then, optimal value function satisfies:
 \[
 V(s,g) \ge V(s',g) - [c(s,g)\Delta t +o(\Delta t)]
 \]
 \label{lemma:local_dyn_constrant}
\end{lemma}

\begin{lemma}
Where $V(\cdot,g)$ is classically differentiable, the first-order Taylor expansion \cite{giammarino2025physics} yields:
$$V(s',g) = V(s,g) + \nabla_s V(s,g)^\top f(s,a)\Delta t + o(\Delta t)$$
\label{lemma:local_taylor_expansion}
\end{lemma}

\begin{lemma}
    Substituting \ref{lemma:local_taylor_expansion} in \ref{lemma:local_dyn_constrant}, dividing by $\Delta t$, and taking the limit as $\Delta t \to 0$, the value function satisfies the inequality for all $a \in \mathcal{A}$:
    $$0 \le \nabla_s V(s,g)^\top f(s,a) - c(s,g)$$
    \label{lemma:hjb_inequality}
\end{lemma}

\subsection{Integral Relaxation}
$\nabla_s V$ can be difficult to optimize at shock boundaries \cite{cannarsa2004semiconcave} where optimal characteristic curves intersect. Consequently, the Taylor expansion (Lemma \ref{lemma:local_taylor_expansion}) and the formulation of the HJB (Theorem 1) can break down at these boundaries \cite{crandall1983viscosity, bardi1997optimal}. 

To improve stability under such regimes, we replace the strict pointwise constraint with an integral formulation defined over a spatial measure. This allows us to define a simple smoothed, one-sided, locally-averaged triangle-inequality based objective. 

We draw the readers' attention to recent work on Hamilton-Jacobi equations over probability spaces and Wasserstein geometry \cite{gangbo2021finite, gangbo2008hamilton, badreddine2022solutions}, where weak notions of solutions are introduced for non-smooth optimal control problems defined over measure-valued states. Unlike these approaches, however, we do not formulate the control problem with a lifted infinite-dimensional HJB equation. Instead, we retain the original Euclidean state space and introduce a localized mollifier operator that acts directly on the value function.

\paragraph{Assumption 5 (Mollifier) } For each state $\vs \in \mathcal{S}$, let $\rho_\delta(s'|s)$ denote a localized probability density parameterized by a bandwidth $\delta>0$. Furthermore, $\int_\delta \rho_\delta(s'|s)d\vs' = 1$ and $\rho_\delta(\cdot|s)$ possesses finite second moments. The density $\rho_\delta(s'|s)$ defines a localized weighting over nearby admissible transitions and acts as a smoothing kernel over the local geometry of the state space. 

Such localized averaging constructions are inspired by approximate supersolution and mollification techniques arising in stochastic Hamilton--Jacobi theory, where local spatial averaging improves regularity while preserving large-scale solution structure.

Recall the local consistency constraint derived from Lemma \ref{lemma:local_dyn_constrant}. Assuming locally uniform traversal velocity, the incremental running cost scales proportionally with spatial displacement. Thus, for sufficiently small neighborhoods, we let $c(s,g)\Delta t\propto  c(s,g)\|s'-s\|$. 

\begin{lemma}
Let $W^{1,\infty}(\mathcal S)
:=
\left\{
u \in L^\infty(\mathcal S)
\;\middle|\;
D^\alpha u \in L^\infty(\mathcal S),
\ |\alpha| \le 1
\right\}
$ denote the Sobolev space of essentially bounded functions with essentially bounded weak first derivatives, and let $L^\infty(\mathcal S)
:=
\left\{
u : \mathcal S \to \mathbb R
\;\middle|\;
\|u\|_{L^\infty(\mathcal S)} < \infty
\right\}$  denote the Banach space of essentially bounded measurable functions equipped with the norm
$
\|u\|_{L^\infty(\mathcal S)}
=
\operatorname*{ess\,sup}_{s \in \mathcal S} |u(s)|.
$
We define the residual operator
$\mathcal R_\delta :
W^{1,\infty}(\mathcal S)
\to
L^\infty(\mathcal S)
$ by
$$
(\mathcal R_\delta[V])(s,g)
:=
\int
\Phi(V,s,s',g)
\,
\rho_\delta(s'|s)
\,ds',
$$ 
where the residual
$\Phi$
is defined as
$\Phi(V,s,s',g):=V(s',g)-
V(s,g)-c(s,g)\|s'-s\|.$
\end{lemma}

Under Assumptions 1--4, the operator $\mathcal R_\delta$ is well-defined for all
$V \in W^{1,\infty}(\mathcal S)$.
The parameter $\delta$ controls the locality of the averaging operator, while
$\rho_\delta(s'|s)$ acts as a mollifying kernel over the local state geometry.

\begin{theorem}
Assume Assumptions 1--4 hold, and suppose the value function $V(\cdot,g)$ is locally Lipschitz continuous with constant $L_V$. Furthermore, let the localized kernel $\rho_\delta(s'|s) \ge 0$ satisfy $\int_{\mathcal S} \rho_\delta(s'|s)\,ds' = 1$, and assume its first spatial moment is bounded by $\int_{\mathcal S} \|s'-s\| \rho_\delta(s'|s)\,ds' \le C_\rho \delta$, for some constant $C_\rho > 0$ independent of $\delta$. 

Then the residual operator $\mathcal R_\delta$ satisfies $\|\mathcal R_\delta[V]\|_{L^\infty(\mathcal S)} \le (L_V + 1)C_\rho \delta$. 
\end{theorem}

\begin{proof}
By the local Lipschitz continuity of $V(\cdot,g)$, $$|V(s',g)-V(s,g)| \le L_V \|s'-s\|$$
From Assumption 3 ($0 \le c(s,g)\le 1$), the residual $\Phi(V,s,s',g) = V(s',g)-V(s,g) - c(s,g)\|s'-s\|$ is bounded by $|\Phi(V,s,s',g)| \le (L_V+1)\|s'-s\|$. 

Integrating against the non-negative localized kernel $\rho_\delta(s'|s)$ gives:
\[
\begin{aligned}
|(\mathcal R_\delta[V])(s,g)| &\le \int_{\mathcal S} |\Phi(V,s,s',g)| \rho_\delta(s'|s) \,ds' \\
&\le (L_V+1) \int_{\mathcal S} \|s'-s\| \rho_\delta(s'|s) \,ds' \\
&\le (L_V+1)C_\rho \delta.
\end{aligned}
\]

\end{proof}

\begin{theorem}
\label{thm:expected_quasimetric}
Let $\Phi(V,s,s',g):=V(s',g)-V(s,g)-c(s,g)|s'-s|,$
and define the mollified residual operator
\[
(\mathcal R_\delta[V])(s,g)
=
\int_{\mathcal S}
\Phi(V,s,s',g)\rho_\delta(s'|s),ds',
\]
where $\rho_\delta(\cdot|s)$ is a non-negative normalized kernel satisfying
$\rho_\delta(s'|s)\ge0, 
\int_{\mathcal S}\rho_\delta(s'|s),ds'=1.$ Further define the positive and negative residual components
$
\Phi^+(V,s,s',g)
:=
\max(\Phi(V,s,s',g),0),$
$
\Phi^-(V,s,s',g)
:=
\max(-\Phi(V,s,s',g),0),
$ such that $
\Phi
=
\Phi^+-\Phi^-.$ Suppose the value function satisfies the mollified consistency condition
$(\mathcal R_\delta[V])(s,g)\le0.$

Then the expected quasimetric violation is bounded by the expected contractive component:
\[
\mathbb E_{s'\sim\rho_\delta(\cdot|s)}
\left[
\Phi^+(V,s,s',g)
\right]
\le
\mathbb E_{s'\sim\rho_\delta(\cdot|s)}
\left[
\Phi^-(V,s,s',g)
\right].
\]

Equivalently,
\[
\mathbb E[\Phi^+]
-
\mathbb E[\Phi^-]
\le0.
\]

Consequently, the mollified operator preserves quasimetric consistency in expectation. 
\end{theorem}

\begin{proof}
Since
$\Phi^+,\Phi^- \ge 0$ pointwise by construction, and $\Phi = \Phi^+-\Phi^-,$ the mollified residual can be decomposed as
\[
\begin{aligned}
(\mathcal R_\delta[V])(s,g)
&=
\int_{\mathcal S}
\left(
\Phi^+(V,s,s',g)
-
\Phi^-(V,s,s',g)
\right)
\rho_\delta(s'|s),ds' \\
&=
\int_{\mathcal S}
\Phi^+(V,s,s',g)\rho_\delta(s'|s),ds'
-
\int_{\mathcal S}
\Phi^-(V,s,s',g)\rho_\delta(s'|s),ds'.
\end{aligned}
\]

Using the normalization condition $\int_{\mathcal S}\rho_\delta(s'|s),ds'=1,$, 
the two integrals correspond to expectations under the local spatial measure induced by $\rho_\delta$:
\[
(\mathcal R_\delta[V])(s,g)
=
\mathbb E_{s'\sim\rho_\delta}
[\Phi^+]
-
\mathbb E_{s'\sim\rho_\delta}
[\Phi^-].
\]

By assumption,
\[
(\mathcal R_\delta[V])(s,g)\le0.
\]

Therefore,
\[
\mathbb E_{s'\sim\rho_\delta}
[\Phi^+]
-
\mathbb E_{s'\sim\rho_\delta}
[\Phi^-]
\le0,
\]
which immediately yields
\[
\mathbb E_{s'\sim\rho_\delta}
[\Phi^+]
\le
\mathbb E_{s'\sim\rho_\delta}
[\Phi^-].
\]

Hence, although individual neighboring states may violate the pointwise quasimetric inequality $V(s',g)-V(s,g)\le c(s,g)|s'-s|,$
the aggregate expansive behavior remains bounded by the aggregate contractive behavior over the local spatial neighborhood. The mollified operator therefore induces a weak expected quasimetric consistency condition rather than a strict pointwise constraint.

Finally, since
$\Phi^+(V,s,s',g)\ge0$,
Markov's inequality implies that for any threshold $\tau>0$,
\[
\mathbb P_{s'\sim\rho_\delta}
\left(
\Phi(V,s,s',g)>\tau
\right)
=
\mathbb P_{s'\sim\rho_\delta}
\left(
\Phi^+(V,s,s',g)>\tau
\right)
\le
\frac{
\mathbb E_{s'\sim\rho_\delta}[\Phi^+]
}{\tau}.
\]

Therefore, the expected quasimetric defect controls the probability of large local quasimetric violations under the spatial kernel $\rho_\delta$.
\end{proof}
\begin{lemma}[Local Curvature Characterization]
\label{lemma:laplacian_characterization}
Assume $V(\cdot,g) \in C^2(\mathcal{S})$ and let the mollified residual operator be defined as
\[
(\mathcal{R}_\delta[V])(s,g) = \int \Phi(V,s,s',g) \rho_\delta(s'|s) \, ds',
\]
where $\Phi(V,s,s',g) = V(s',g) - V(s,g) - c(s,g)|s'-s|$. Suppose the kernel $\rho_\delta(\cdot|s)$ is isotropic with zero first moment and covariance
\[
\int (s'-s)(s'-s)^\top \rho_\delta(s'|s) \, ds' = \sigma^2\delta^2 I.
\]
Then the mollified residual admits the expansion
\[
(\mathcal{R}_\delta[V])(s,g) = \frac{\sigma^2\delta^2}{2}\Delta V(s,g) - c(s,g)\mu_1\delta + o(\delta^2),
\]
where $\mu_1 = \int |s'-s| \rho_\delta(s'|s) \, ds'$, and $\Delta V(s,g) = \mathrm{tr}(\nabla_s^2V(s,g))$ denotes the Laplacian. Consequently, Although the leading-order behavior is governed by the local quasimetric constraint, isotropic averaging induces a higher-order diffusion-like regularization effect.
\end{lemma}

\begin{proof}
Let $\xi = s'-s$. Applying a second-order Taylor expansion of $V$ around $s$ gives:
\begin{align*}
V(s+\xi,g) &= V(s,g) + \nabla_sV(s,g)^\top\xi + \frac{1}{2}\xi^\top\nabla_s^2V(s,g)\xi + o(|\xi|^2), \\
\Phi(V,s,s+\xi,g) &= \nabla_sV(s,g)^\top\xi + \frac{1}{2}\xi^\top\nabla_s^2V(s,g)\xi - c(s,g)|\xi| + o(|\xi|^2).
\end{align*}
Integrating against $\rho_\delta$ gives $(\mathcal{R}_\delta[V])(s,g) = I_1 + I_2 - I_3 + o(\delta^2)$, where the component integrals are:
\[
I_1 = \int \nabla_sV(s,g)^\top\xi \, \rho_\delta(\xi) \, d\xi, \quad
I_2 = \frac{1}{2} \int \xi^\top\nabla_s^2V(s,g)\xi \, \rho_\delta(\xi) \, d\xi, \quad
I_3 = c(s,g) \int |\xi| \rho_\delta(\xi) \, d\xi.
\]
Since the kernel is isotropic with zero first moment, $\int \xi \rho_\delta(\xi) \, d\xi = 0$, and therefore $I_1 = 0$. Using the trace identity $\xi^\top H\xi = \mathrm{tr}(H\xi\xi^\top)$, we can evaluate $I_2$:
\[
I_2 = \frac{1}{2} \mathrm{tr} \left( \nabla_s^2V(s,g) \int \xi\xi^\top \rho_\delta(\xi) \, d\xi \right) = \frac{\sigma^2\delta^2}{2} \Delta V(s,g),
\]
where we substituted the assumed covariance condition. Finally, setting $I_3 = c(s,g)\mu_1\delta$ and combining the evaluated terms yields the final expansion:
\[
(\mathcal{R}_\delta[V])(s,g) = \frac{\sigma^2\delta^2}{2}\Delta V(s,g) - c(s,g)\mu_1\delta + o(\delta^2).
\]
\end{proof}

\paragraph{Interpretation.}
Theorem~\ref{thm:expected_quasimetric} formalizes the role of the mollified objective as a spatial regularizer over the value geometry rather than a dynamics-valid local transition model. In particular, the isotropic kernel $\rho_\delta$ does not enforce directional feasibility or admissible system trajectories. Instead, it defines a local averaging measure over the ambient state geometry, analogous to classical mollifiers used in PDE regularization.

Under this interpretation, the temporal-difference objective provides directional Bellman supervision along dataset transitions, while the mollified residual controls undesirable local geometric pathologies such as sharp oscillations, non-smooth value discontinuities, and locally expansive artifacts that may arise in sparse offline datasets. The resulting formulation therefore acts as a weak quasimetric regularizer in expectation, preserving globally contractive value geometry without requiring pointwise directional constraints or explicit differential operators.





\paragraph{Discrete Empirical Approximation}

We discretize the integral and use a sum to compute the residual. Specifically, for each point $s \in \mathcal{S}$, let $s'_i \sim \rho_\delta(\cdot|s)$ be the transition sampled from the local state space. We randomly sample adjacent points $\{s'_1, \dots, s'_N\} \sim \mathcal{N}(\vs, \delta^2I)$ to form the local neighborhood. We approximate the integral as:
$$\mathcal{R}_\delta[V](s,g) \approx \frac{1}{N} \sum_{i=1}^N \big[ V(s'_i,g) - V(s,g) - c(s,g)\|s'_i - s\| \big]$$

\section{Notations}
\label{appx: notation}

To promote readability, we provide a table of notations used in~\cref{tab: notations}.
\begin{table}[!ht]
    \centering
    
    \begin{tabular}{@{}cl@{}}
        \toprule
        \textbf{Notations} & \textbf{Explanation} \\
        \midrule
        \multicolumn{2}{l}{\textbf{Goal-conditional Reinforcement Learning}} \\
        $\gM$ & A Markov decision process (MDP). \\
        $\gS$ & State space of a Markov decision process. \\
        $\gA$ & Action space of a Markov decision process. \\ 
        $\vs_{t}$ & State at time step $t$. \\
        $\vs^{\prime}$ & Generic notation for next state given current state $\vs$. \\
        $\va_{t}$ & Action of the agent at time step $t$. \\
        $r(\vs_{t},\vg)$ & Scalar goal-conditioned reward at state $s_{t}\in\gS$ to reach goal $g\in\gS$. \\
        $p(\vs_{t+1}\mid{\vs_{t},\va_{t}})$ & System dynamics of an MDP as a state transition probability. \\
        $\pi(\va_{t}\mid{\vs_{t},g})$ & Goal-conditioned policy. \\
        $\Pi$ & The space of valid policies. \\
        $\bm{\tau}$ & Trajectory consisting a sequence of state-action pairs $\bm{\tau}\triangleq(\vs_{0},\va_{0},\ldots,\vs_{T},\va_{T})$. \\
        $p^{\pi}(\bm{\tau}\mid{\cdot})$ & A distribution of trajectories conditioned on policy $\pi$.\\
        $V^{\pi}(\vs_{t},\vg)$ & Goal-conditioned state value function following policy $\pi$. \\
        $V^{\ast}(\vs_{t},\vg)$ & Optimal Goal-conditioned state value function following policy $\pi^{\ast}$. \\
        \midrule
        \multicolumn{2}{l}{\textbf{Optimal Control}} \\
        $f(\vs,\va)$ & Deterministic component in the state transition dynamics. \\
        $\delta$ & Scalar radius factor of the jump. \\
        $\mG(\vs)$ & Control matrix at state $\vs$. \\
        $H(\vs,\va,\cdot)$
        & Hamiltonian of the control system. \\
        $c(\vs,\va)$ & Running cost in the control system. \\
        $q(\vs,\cdot)$ & Terminal cost in the control system. \\
        \bottomrule
    \end{tabular}
    \caption{Table of notations in the paper.}
    \label{tab: notations}
\end{table}

\section{Implementation and Hyperparameters}

\label{appx: impl}

\begin{figure*}[t]
    \centering
    \setlength{\tabcolsep}{1pt}
    
    \begin{subfigure}[b]{0.24\textwidth}
        \centering
        \includegraphics[width=\linewidth]{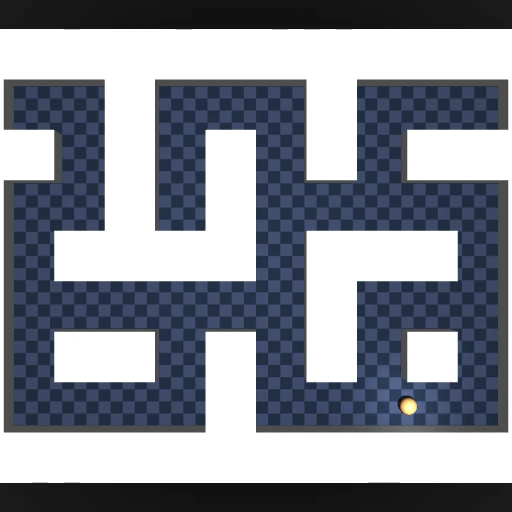}
        \caption{\texttt{point-maze}}
    \end{subfigure}
    \hfill
    \begin{subfigure}[b]{0.24\textwidth}
        \centering
        \includegraphics[width=\linewidth]{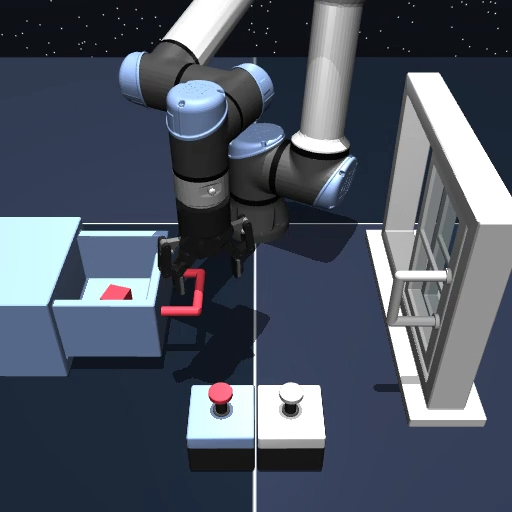}
        \caption{\texttt{scene-play}}
    \end{subfigure}
    \hfill
    \begin{subfigure}[b]{0.24\textwidth}
        \centering
        \includegraphics[width=\linewidth]{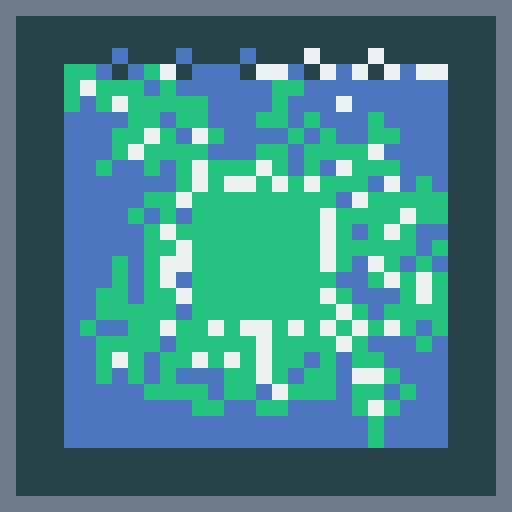}
        \caption{\texttt{powderworld-hard}}
    \end{subfigure}
    \hfill
    \begin{subfigure}[b]{0.24\textwidth}
        \centering
        \includegraphics[width=\linewidth]{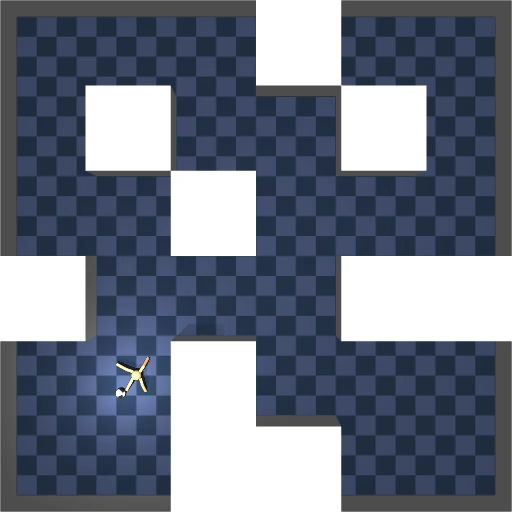}
        \caption{\texttt{ant-soccer}}
    \end{subfigure}
    
    \vspace{0.2cm} 
    
    \begin{subfigure}[b]{0.24\textwidth}
        \centering
        \includegraphics[width=\linewidth]{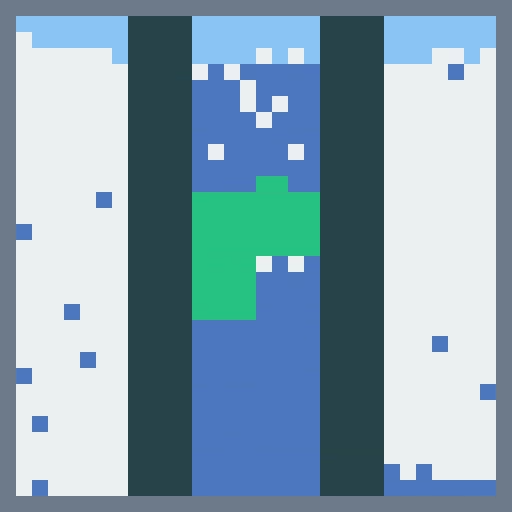}
        \caption{\texttt{pw-hard-3-rooms}}
    \end{subfigure}
    \hfill
    \begin{subfigure}[b]{0.24\textwidth}
        \centering
        \includegraphics[width=\linewidth]{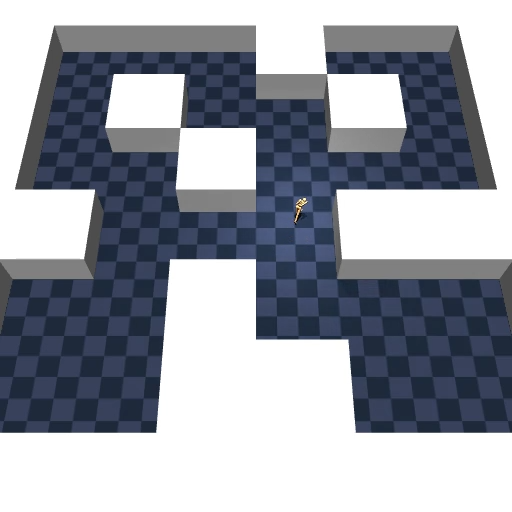}
        \caption{\texttt{humanoid-maze}}
    \end{subfigure}
    \hfill
    \begin{subfigure}[b]{0.24\textwidth}
        \centering
        \includegraphics[width=\linewidth]{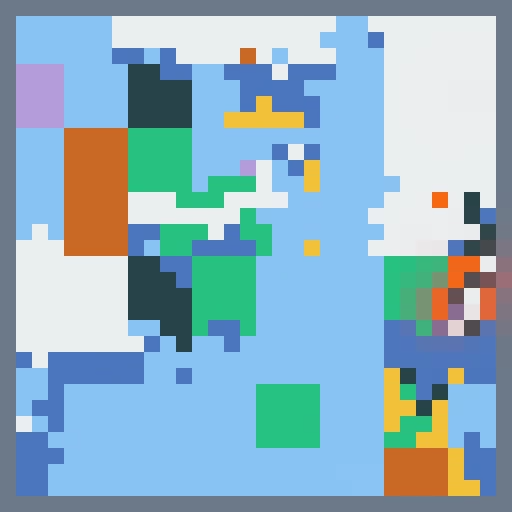}
        \caption{\texttt{powderworld-play}}
    \end{subfigure}
    \hfill
    \begin{subfigure}[b]{0.24\textwidth}
        \centering
        \includegraphics[width=\linewidth]{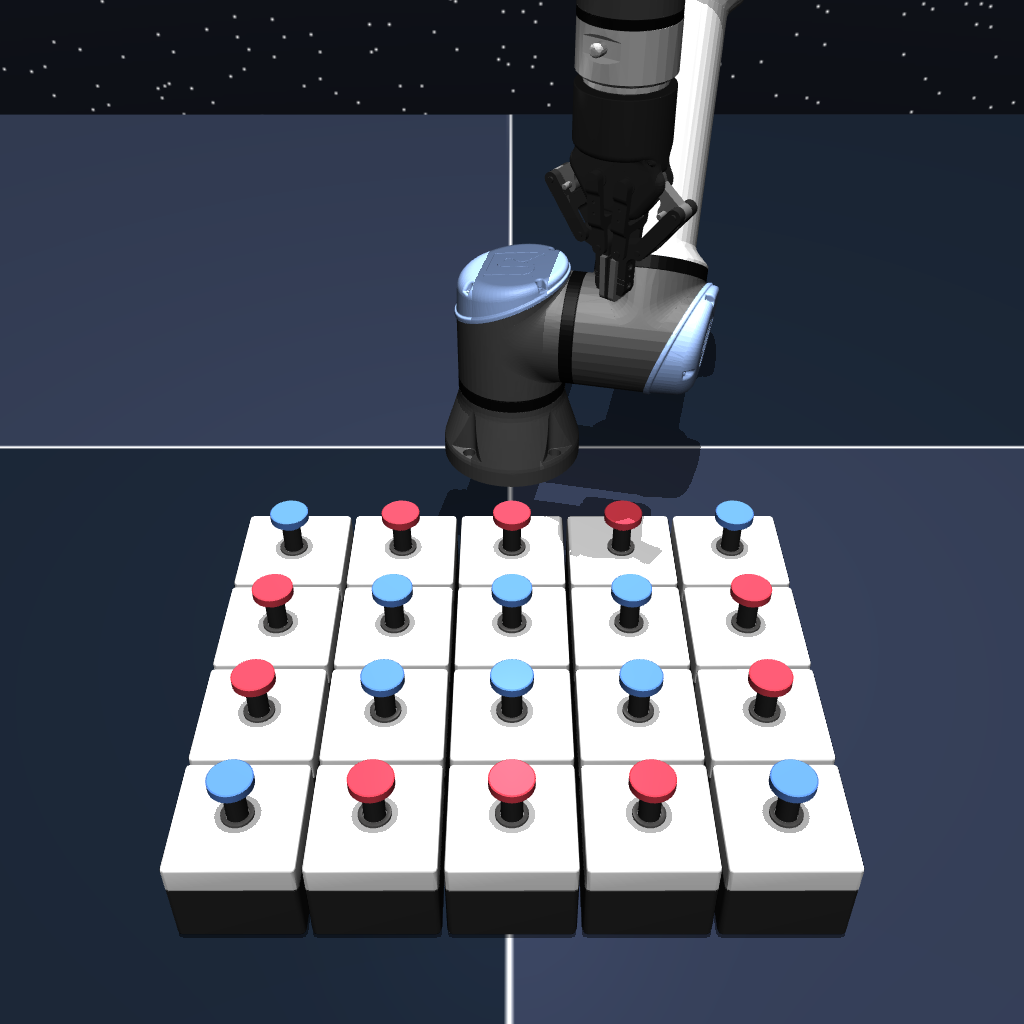}
        \caption{\texttt{puzzle-4x5}}
    \end{subfigure}
    \caption{\textbf{Benchmark Suite from \texttt{OGBench}} We evaluate our method across a wide spectrum of offline GCRL tasks, ranging from standard geometric navigation (\texttt{point/ant/humanoid-maze}) and high-dimensional manipulation (\texttt{scene-play, puzzle}), to the highly stochastic, pixel-based physics of \texttt{powderworld}. This diversity tests the agent's ability to handle varying degrees of state dimensionality, transition stochasticity, and dynamic complexity.}
    \label{fig:environments_grid}
\end{figure*}

Our experimental framework is built upon the OGBench benchmark suite \cite{park2024ogbench}. 
For the dual representation baselines, we follow \cite{park2025dualgoalrepresentations} to compute temporal distances by randomly sampling 64 anchor states and executing a breadth-first search (BFS). The vector of distances to these anchors serves as the dual representation. We train a goal-conditioned DQN using a sparse reward function ($-1$ for non-goal states, $0$ for the goal). Agents are trained for $10^6$ steps, and the argmax policy is evaluated over 15 episodes per task. A full list of hyperparameters, including the specific hindsight relabeling ratios (curriculum, geometric, trajectory, and random), is provided in Table 8.

\noindent\textbf{Reward.}
As per OGBench, we define a sparse binary reward function for training GCVF: $r(\vs, \vg) = 0$ if $\vs=\vg$, and $-1$ otherwise. Under this formulation, the optimal value function encodes the discounted temporal distance:
\begin{equation*}
    V^*(\vs, \vg) = -\frac{1 - \gamma^{d^*(s, g)}}{1 - \gamma}
\end{equation*}
Following the dual-goal-representations~\cite{park2025dualgoalrepresentations}, we approximate this value using the inner product of the learned representations, $f(\psi(\vs), \phi(\vg)) = \psi(\vs)^\top \phi(\vg)$. 

\noindent\textbf{Evaluation.}
We report performance based on the average success rate over the final three training epochs. This corresponds to the window covering 800K–1M gradient steps. Evaluation aggregates results from 50 episodes (state-based) across the five designated evaluation goals.
\begin{figure}[b!]
    \centering
    \begin{subfigure}[b]{0.49\textwidth}
        \centering
        \includegraphics[width=\linewidth]{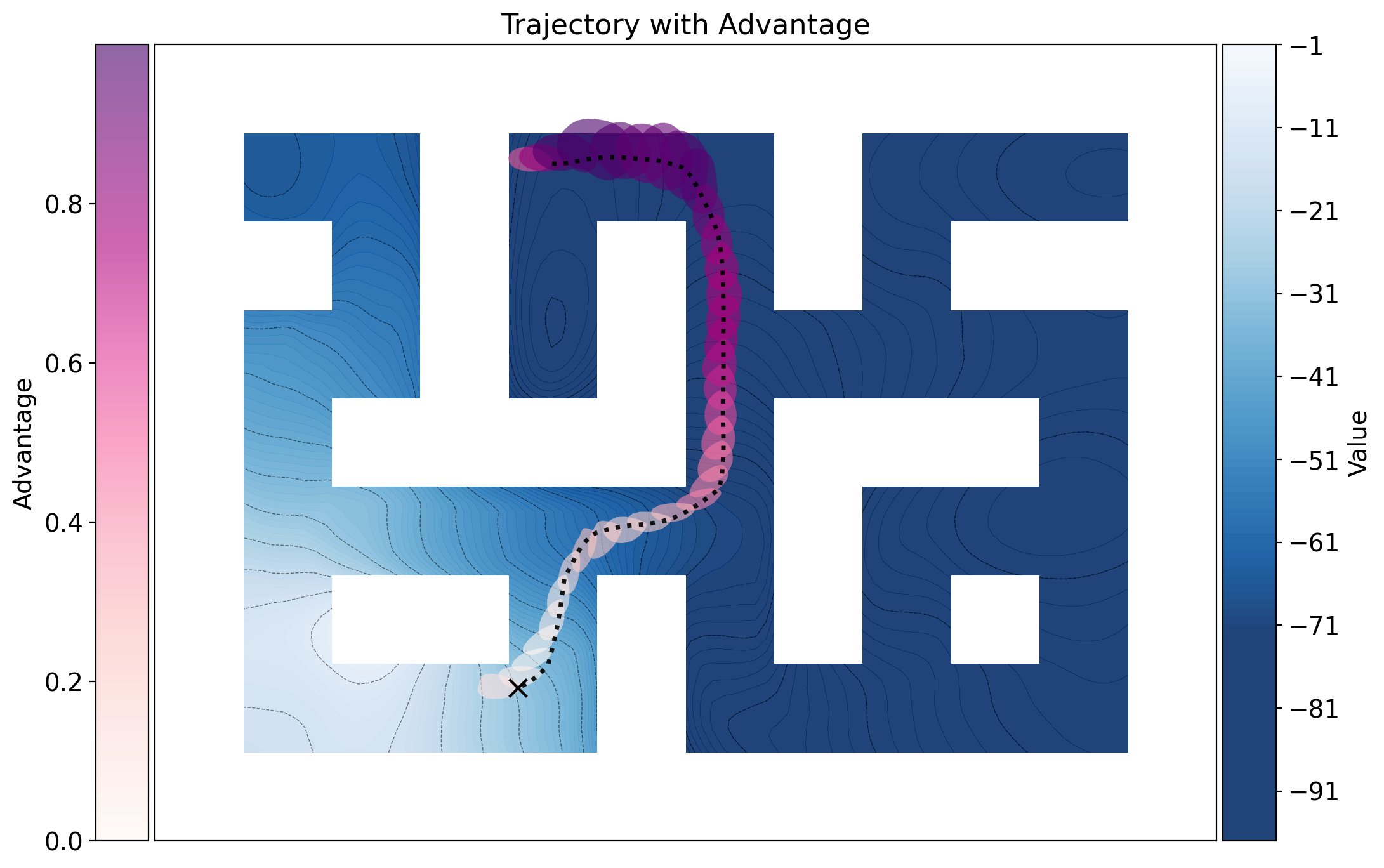}
        \caption{Ours}
        \label{fig:ours_advantage}
    \end{subfigure}
    \hfill 
    \begin{subfigure}[b]{0.49\textwidth}
        \centering
        \includegraphics[width=\linewidth]{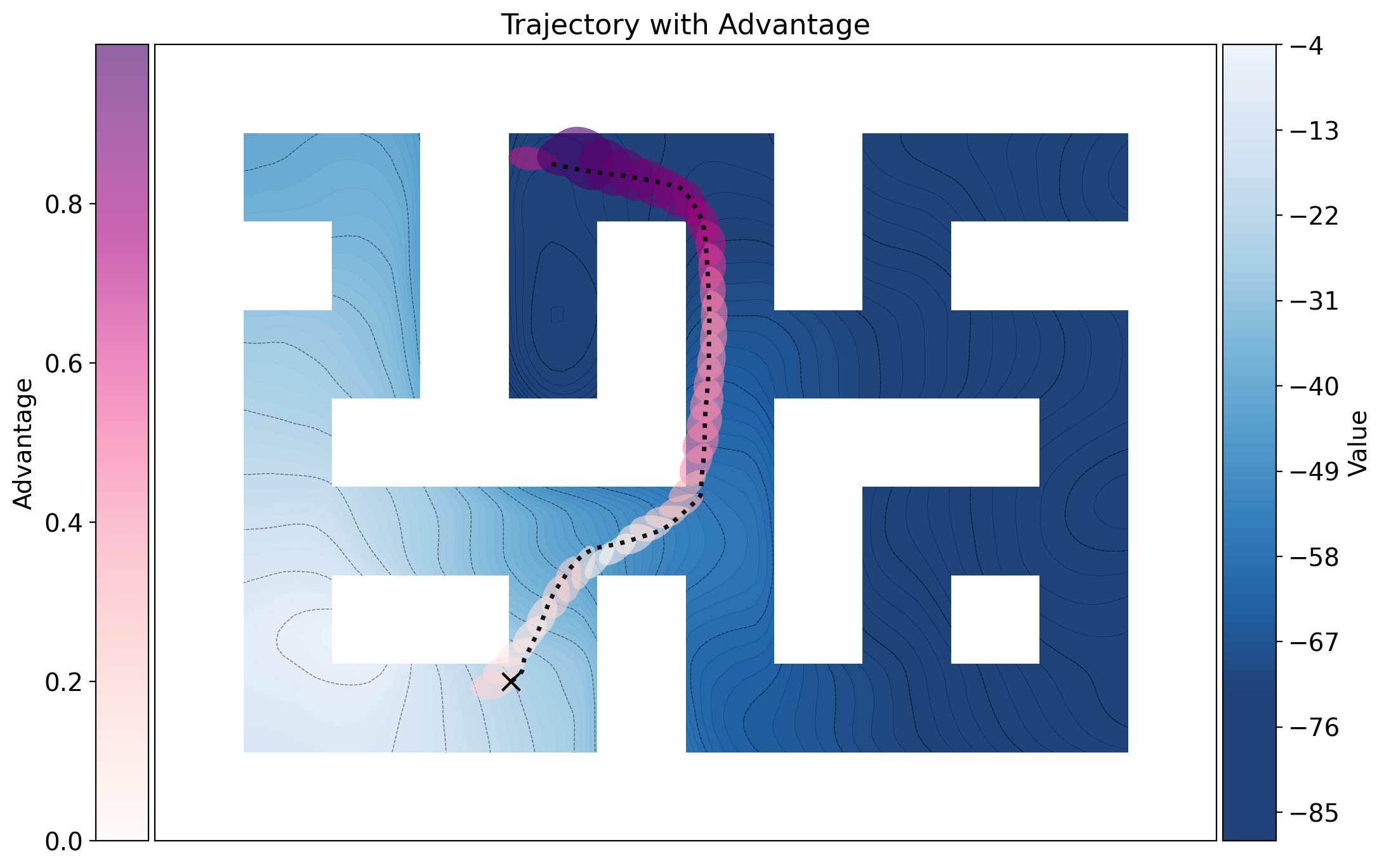}
        \caption{Eikonal Baseline}
        \label{fig:eikonal_advantage}
    \end{subfigure}

    \caption{\textbf{Comparison of action probability distributions (pink fans) along a trajectory.} \textbf{(a) Ours:} The advantage distributions align parallel to the walls, correctly identifying the safe path through corridors. \textbf{(b) Eikonal:} The baseline shows similar geometry trends, highlighting that our approach models the geometry without relying on expensive and unstable first-order \texttt{autograd} operations.}
    \label{fig:advantage_comparison}
    \vspace{-1em}
\end{figure}

\noindent\textbf{Hyperparameters and Architecture.}
Complete hyperparameter configurations are listed in Table 9. Consistent with OGBench, we apply Layer Normalization to all neural network layers, with the exception of the policy networks in GCIVL and CRL baselines. We use the Adam optimizer and GELU activation functions throughout. 

\begin{table}[!ht]
\centering
\caption{Hyperparameters for experiments.}
\label{tab:hyperparams_ogbench}
\resizebox{\columnwidth}{!}{%
\begin{tabular}{ll}
\toprule
\textbf{Hyperparameter} & \textbf{Value} \\
\midrule
Gradient steps & $10^6$ (state-based), $5 \times 10^5$ (pixel-based) \\
Optimizer & Adam \\
Learning rate & $0.0003$ \\
Batch size & $1024$ (state-based), $256$ (pixel-based) \\
MLP size & $[512, 512, 512]$ \\
Nonlinearity & GELU \\
Target network update rate & $0.005$ \\
Discount factor $\gamma$ & $0.99$ (default), $0.995$ (humanoidmaze, antmaze-giant) \\
Goal representation dimensionality $N$ & $256$ (default), $64$ (pointmaze) \\
Dual representation's GCIQL expectile $\kappa$ & $0.7$ (default), $0.9$ (navigate) \\
Size of random walk neighborhood $K$ & 10 \\
Eikonal Speed Type & Constant, exactly defined in \cite{giammarino2025physics} \\
Noise standard deviation $\nu$ & 0.01 \\
VIB $\beta$ & $0.001$ (default), $0.003$ (pointmaze, antmaze-large) \\
Rep $(p_{\mathcal{D}}^{\text{cur}}, p_{\mathcal{D}}^{\text{geom}}, p_{\mathcal{D}}^{\text{traj}}, p_{\mathcal{D}}^{\text{rand}})$ ratio (Dual) & $(0.2, 0.5, 0, 0.3)$ \\
Rep $(p_{\mathcal{D}}^{\text{cur}}, p_{\mathcal{D}}^{\text{geom}}, p_{\mathcal{D}}^{\text{traj}}, p_{\mathcal{D}}^{\text{rand}})$ ratio (TRA, BYOL-$\gamma$) & $(0, 1, 0, 0)$ \\
Downstream GCIVL expectile & $0.9$ \\
Downstream policy extraction hyperparameters & Identical to OGBench\\
Downstream policy $(p_{\mathcal{D}}^{\text{cur}}, p_{\mathcal{D}}^{\text{geom}}, p_{\mathcal{D}}^{\text{traj}}, p_{\mathcal{D}}^{\text{rand}})$ ratio & $(0, 0, 1, 0)$ \\
Downstream value $(p_{\mathcal{D}}^{\text{cur}}, p_{\mathcal{D}}^{\text{geom}}, p_{\mathcal{D}}^{\text{traj}}, p_{\mathcal{D}}^{\text{rand}})$ ratio (default) & $(0.2, 0.5, 0, 0.3)$ \\
Downstream value $(p_{\mathcal{D}}^{\text{cur}}, p_{\mathcal{D}}^{\text{geom}}, p_{\mathcal{D}}^{\text{traj}}, p_{\mathcal{D}}^{\text{rand}})$ ratio (CRL) & $(0, 1, 0, 0)$ \\
\bottomrule
\end{tabular}%
}
\end{table}

\begin{table}[h!]
\centering
\caption{\textbf{HYPERPARAMETERS SUMMARY.} Base algorithmic and representation-specific hyperparameters across all evaluated environments. All Dual models utilize a \texttt{bilinear} representation type. Visual environments additionally apply an \texttt{impala\_small} encoder, a batch size of 256, $p_{\text{aug}} = 0.5$, and are trained for 500K steps with 15 evaluation episodes.}
\label{tab:hyperparameters}
\resizebox{\textwidth}{!}{%
\begin{tabular}{lcccccc}
\toprule
\textbf{\texttt{ENVIRONMENT DOMAIN}} & \textbf{\texttt{Discount ($\gamma$)}} & \textbf{\texttt{GCIVL $\alpha$}} & \textbf{\texttt{CRL $\alpha$}} & \textbf{\texttt{Rep Dim}} & \textbf{\texttt{DUAL $\tau$}} & \textbf{\texttt{VIB $\beta$}} \\
\midrule
\texttt{pointmaze} (medium/large)     & 0.99  & 10.0 & 0.03 & 64  & 0.9 & 0.003 \\
\texttt{antmaze-medium}               & 0.99  & 10.0 & 0.1  & 256 & 0.9 & 0.001 \\
\texttt{antmaze-large}                & 0.99  & 10.0 & 0.1  & 256 & 0.9 & 0.003 \\
\texttt{humanoidmaze} (medium/large)  & 0.995 & 10.0 & 0.1  & 256 & 0.9 & 0.001 \\
\texttt{antsoccer-arena}              & 0.99  & 10.0 & 0.3  & 256 & 0.9 & 0.001 \\
\texttt{manip/play} (cube, scene, puzzle, D4RL) & 0.99 & 10.0 & 3.0  & 256 & 0.7 & 0.001 \\
\bottomrule
\end{tabular}%
}
\end{table}


       
       
       

\section{Implementation Details}
The framework is implemented using JAX and Flax. A primary advantage of formulating the value geometry as a spatial expectation is its native compatibility with modern hardware accelerators (GPUs/TPUs). Because the mollified objective evaluates a local state distribution rather than explicit spatial derivatives, it avoids the sequential bottlenecks and graph-unrolling overhead typical of automatic differentiation for continuous constraints. By leveraging JAX's inherently vectorized operations, evaluating the $K$ surrounding neighborhood states requires only a single batched forward pass through the value network. Consequently, expanding the local sample size $K$ introduces negligible computational overhead compared to the standard offline RL updates.

To facilitate reproducibility, we provide a minimal, algorithm of our mollified geometric objective. 

\begin{algorithm}[h!]
\caption{Mollified Value Learning (MVL) Objective}
\label{alg:mvl_objective}
\begin{algorithmic}[1]
\Require Value network $V_\theta$, batch of states $\mathcal{B}_s$, batch of goals $\mathcal{B}_g$, spatial variance $\nu$, neighborhood size $K$, local running cost $c_s=1.0$
\State $\sigma \gets \sqrt{2\nu}$
\State Initialize loss $\mathcal{L}_{\text{MVL}} \gets 0$
\For{each state-goal pair $(\vs, \vg) \in (\mathcal{B}_s, \mathcal{B}_g)$}
    \State \text{Compute anchor value:} \quad $v_s \gets V_\theta(\vs, \vg)$
    \For{$k = 1 \dots K$}
        \State \text{Sample spatial perturbation:} \quad $\boldsymbol{\epsilon}_k \sim \mathcal{N}(\mathbf{0}, \sigma^2 \mathbf{I})$
        \State \text{Generate local neighbor:} \quad $\vs'_k \gets \vs + \boldsymbol{\epsilon}_k$
        \State \text{Evaluate neighbor value:} \quad $v'_k \gets V_\theta(\vs'_k, \vg)$
        \State \text{Compute spatial distance:} \quad $d_k \gets \|\boldsymbol{\epsilon}_k\|_2$
        \State \text{Compute value variation:} \quad $\Delta v_k \gets v'_k - v_s$
        \State \text{Evaluate shortest-path residual:} \quad $r_k \gets \max\big(0, \Delta v_k - c_s \cdot d_k\big)$
    \EndFor
    \State \text{Aggregate over local spatial distribution:} \quad $\mathcal{L}_{\text{MVL}} \gets \mathcal{L}_{\text{MVL}} + \frac{1}{K} \sum_{k=1}^K r_k^2$
\EndFor
\State \Return $\mathcal{L}_{\text{MVL}} / |\mathcal{B}_s|$
\end{algorithmic}
\end{algorithm}

\section{Ablations}

\begin{table}[h!]
\centering
\caption{\textbf{HIGH-DIMENSIONAL STOCHASTIC DYNAMICS (POWDERWORLD).} Overall results on pixel-based tasks where environment dynamics are highly stochastic (e.g., sand piling, fire spreading). Since raw pixel perturbation is infeasible, \texttt{GCIVL-PIXEL+MVL} enforces the viscous consistency condition via sampling in the \textit{Impala} encoder's latent embedding space.}
\label{tab:powderworld_results}
\resizebox{0.9\columnwidth}{!}{%
\begin{tabular}{lccc}
\toprule
\textbf{\texttt{ENVIRONMENT}} & \textbf{\texttt{GCIVL-PIXEL}} & \textbf{\texttt{GCIVL-PIXEL-EIK}} & \textbf{\texttt{GCIVL-PIXEL+MVL (OURS)}} \\
\midrule
\texttt{powderworld-easy}   & \bt{99 $\pm$ 1} & \bt{99 $\pm$ 1} & \bt{99 $\pm$ 1} \\
\texttt{powderworld-medium} & 50 $\pm$ 4 & \bt{61 $\pm$ 4} & \bt{62 $\pm$ 7} \\
\texttt{powderworld-hard}   & 4 $\pm$ 3 & 7 $\pm$ 3 & \bt{10 $\pm$ 2} \\
\bottomrule
\end{tabular}%
}
\end{table}
\begin{table}[h!]
\centering
\caption{\textbf{ABLATION STUDY ON \texttt{DUAL-MVL}.} We evaluate the success rates across 4 seeds, ablating over the number of sampled states $N$ and the spatial variance $\delta$ across four environments with the largest gains in performance.}
\label{tab:ablation_results_ND}
\resizebox{0.7\textwidth}{!}{%
\begin{tabular}{lccccc}
\toprule
\multicolumn{6}{c}{\textbf{Varying Number of Sampled States ($N$)}} \\
\midrule
\textbf{\texttt{ENVIRONMENT}} & $\mathbf{N=1}$ & $\mathbf{N=5}$ & $\mathbf{N=10}$ & $\mathbf{N=20}$ & $\mathbf{N=25}$ \\
\midrule
\texttt{pointmaze-large} & 57 $\pm$ 3 & \bt{77 $\pm$ 2} & \bt{76 $\pm$ 2} & \bt{77 $\pm$ 2} & \bt{77 $\pm$ 3} \\
\texttt{antmaze-large}   & 28 $\pm$ 6 & 33 $\pm$ 7 & \bt{39 $\pm$ 11}& \bt{37 $\pm$ 6} & \bt{39 $\pm$ 9} \\
\texttt{scene-play}      & 75 $\pm$ 4 & 78 $\pm$ 5 & \bt{84 $\pm$ 5} & \bt{82 $\pm$ 4} & \bt{83 $\pm$ 6} \\
\texttt{puzzle-4x4}      & 26 $\pm$ 3 & \bt{34 $\pm$ 4} & \bt{33 $\pm$ 5} & \bt{34 $\pm$ 2} & \bt{34 $\pm$ 3} \\
\midrule
\multicolumn{6}{c}{\textbf{Varying Spatial Variance ($\delta$)}} \\
\midrule
\textbf{\texttt{ENVIRONMENT}} & $\mathbf{\delta=10^{-4}}$ & $\mathbf{\delta=10^{-3}}$ & $\mathbf{\delta=10^{-2}}$ & $\mathbf{\delta=10^{-1}}$ & $\mathbf{\delta=10^{1}}$ \\
\midrule
\texttt{pointmaze-large} & 55 $\pm$ 2 & 57 $\pm$ 3 & \bt{77 $\pm$ 3} & \bt{79 $\pm$ 2} & 64 $\pm$ 4 \\
\texttt{antmaze-large}   & 24 $\pm$ 7 & 32 $\pm$ 9 & \bt{38 $\pm$ 10} & \bt{39 $\pm$ 11}& 22 $\pm$ 12 \\
\texttt{scene-play}      & 75 $\pm$ 2 & 74 $\pm$ 5 & 79 $\pm$ 4 & \bt{84 $\pm$ 5} & 32 $\pm$ 4 \\
\texttt{puzzle-4x4}      & 25 $\pm$ 2 & 22 $\pm$ 4 & \bt{34 $\pm$ 3} & \bt{34 $\pm$ 2} & 12 $\pm$ 5 \\
\bottomrule
\end{tabular}%
}
\end{table}

In this section we discuss two key ablations for our proposed approach. Our method depends on two key variables - number of neighborhood states $K$ and the radius $\delta$. We present the contour plots to showcase the changes in the learned value function. We present the results in \cref{tab:ablation_hyperparams}. 

As shown in Table~\ref{tab:ablation_hyperparams}, increasing the sample budget $K$ yields a performance improvement from 57\% to 77\% between 1 and 5 samples, after which the success rate saturates. In contrast, the method shows greater sensitivity to the radius ($\delta$). Lower values ($\delta \le 10^{-3}$) result in a marked drop in performance (Table~\ref{tab:ablation_hyperparams}). 


\subsection{Comparing Success Rates on D4RL}
\begin{table*}[h!]
\centering
\caption{\textbf{RESULTS ON D4RL/KITCHEN VARIANTS} We compare the success rates across multiple tasks and environments. \bt{Soft teal} highlights methods within 5\% of the best performance.}
\label{tab:kitchen_dataset_results}
\resizebox{\textwidth}{!}{%
\begin{tabular}{lccccccccc}
\toprule
\textbf{\texttt{ENVIRONMENT/TASK}} & \textbf{\texttt{Orig}} & \textbf{\texttt{Orig-MVL }} & \textbf{\texttt{VIB}} & \textbf{\texttt{VIB-MVL }} & \textbf{\texttt{VIP}} & \textbf{\texttt{TRA}} & \textbf{\texttt{BYOL}} & \textbf{\texttt{DUAL}} & \textbf{\texttt{DUAL-MVL }} \\
\midrule
\multicolumn{10}{l}{\textbf{\texttt{D4RL/Kitchen-Partial}}} \\
\texttt{slide\_cabinet} & \bt{75 $\pm$ 0} & 58 $\pm$ 12 & 58 $\pm$ 12 & \bt{75 $\pm$ 20} & 50 $\pm$ 0 & 50 $\pm$ 0 & 50 $\pm$ 0 & \bt{75 $\pm$ 0} & 50 $\pm$ 0 \\
\texttt{microwave}      & 92 $\pm$ 12 & 92 $\pm$ 12 & 92 $\pm$ 12 & \bt{100 $\pm$ 0} & 75 $\pm$ 0 & \bt{100 $\pm$ 0} & 75 $\pm$ 0 & \bt{100 $\pm$ 0} & \bt{100 $\pm$ 0} \\
\texttt{light\_switch}  & \bt{83 $\pm$ 24} & 75 $\pm$ 20 & 75 $\pm$ 20 & \bt{83 $\pm$ 12} & 50 $\pm$ 0 & 50 $\pm$ 0 & 50 $\pm$ 0 & 67 $\pm$ 12 & 69 $\pm$ 21 \\
\texttt{kettle}         & 75 $\pm$ 0 & \bt{100 $\pm$ 0} & 75 $\pm$ 0 & 83 $\pm$ 12 & 75 $\pm$ 0 & 88 $\pm$ 12 & 75 $\pm$ 0 & 83 $\pm$ 12 & 94 $\pm$ 11 \\
\midrule
\multicolumn{10}{l}{\textbf{\texttt{D4RL/Kitchen-Mixed}}} \\
\texttt{microwave}      & 83 $\pm$ 12 & 83 $\pm$ 12 & \bt{100 $\pm$ 0} & \bt{100 $\pm$ 0} & 75 $\pm$ 0 & \bt{100 $\pm$ 0} & 75 $\pm$ 0 & \bt{100 $\pm$ 0} & \bt{95 $\pm$ 10} \\
\texttt{light\_switch}  & \bt{75 $\pm$ 20} & \bt{75 $\pm$ 20} & \bt{75 $\pm$ 20} & \bt{75 $\pm$ 20} & 50 $\pm$ 0 & 50 $\pm$ 0 & 50 $\pm$ 0 & \bt{75 $\pm$ 20} & 65 $\pm$ 25 \\
\texttt{kettle}         & 75 $\pm$ 0 & 92 $\pm$ 12 & 75 $\pm$ 0 & 92 $\pm$ 12 & 75 $\pm$ 0 & \bt{100 $\pm$ 0} & 75 $\pm$ 0 & 75 $\pm$ 0 & 85 $\pm$ 20 \\
\texttt{bottom\_burner} & 75 $\pm$ 0 & 67 $\pm$ 12 & 75 $\pm$ 20 & 67 $\pm$ 12 & 75 $\pm$ 0 & \bt{88 $\pm$ 12} & 75 $\pm$ 0 & 58 $\pm$ 12 & 60 $\pm$ 20 \\
\midrule
\multicolumn{10}{l}{\textbf{\texttt{D4RL/Kitchen-Complete}}} \\
\texttt{slide\_cabinet} & \bt{25 $\pm$ 0} & 17 $\pm$ 12 & \bt{25 $\pm$ 0} & 8 $\pm$ 12 & 12 $\pm$ 12 & \bt{25 $\pm$ 0} & 0 $\pm$ 0 & 8 $\pm$ 12 & 0 $\pm$ 0 \\
\texttt{microwave}      & 92 $\pm$ 12 & \bt{100 $\pm$ 0} & \bt{100 $\pm$ 0} & \bt{100 $\pm$ 0} & 62 $\pm$ 12 & 61 $\pm$ 4 & 0 $\pm$ 0 & 92 $\pm$ 12 & \bt{100 $\pm$ 0} \\
\texttt{light\_switch}  & 17 $\pm$ 12 & 25 $\pm$ 0 & 25 $\pm$ 0 & 17 $\pm$ 12 & 0 $\pm$ 0 & \bt{38 $\pm$ 12} & 0 $\pm$ 0 & 8 $\pm$ 12 & 0 $\pm$ 0 \\
\texttt{kettle}         & 83 $\pm$ 12 & \bt{92 $\pm$ 12} & 75 $\pm$ 20 & 83 $\pm$ 12 & 25 $\pm$ 0 & 48 $\pm$ 12 & 0 $\pm$ 0 & 75 $\pm$ 20 & 83 $\pm$ 12 \\
\bottomrule
\end{tabular}%
}
\end{table*}

\begin{table*}[t]
\centering
\caption{\textbf{RESULTS ON ORACLEREP VARIANTS} We compare the success rates across multiple tasks and environments. In \texttt{oraclerep} setting, where the goal is denoted by an ``Oracle Representation", which is a subset of state representations. These represent the barebones information required to fulfil the success criteria. \bt{Soft teal} highlights methods within 5\% of the best performance, while \hc{soft purple} highlights MVL when it improves upon IVL by at least 5\%.}
\label{tab:new_dataset_results}
\resizebox{\textwidth}{!}{%
\begin{tabular}{lccccccccc}
\toprule
\textbf{\texttt{ENVIRONMENT/TASK}} & \textbf{\texttt{BC}} & \textbf{\texttt{FBC}} & \textbf{\texttt{CRL}} & \textbf{\texttt{QRL}} & \textbf{\texttt{TDP}} & \textbf{\texttt{COE}} & \textbf{\texttt{TRL}} & \textbf{\texttt{IVL}} & \textbf{\texttt{MVL}} \\
\midrule
\multicolumn{10}{l}{\textbf{\texttt{pointmaze-large-navigate-oraclerep-v0}}} \\
\texttt{task1} & 53 $\pm$ 24 & \bt{98 $\pm$ 3} & 23 $\pm$ 18 & 31 $\pm$ 28 & 40 $\pm$ 19 & 67 $\pm$ 26 & 52 $\pm$ 19 & 94 $\pm$ 6 & \bt{100 $\pm$ 0}\\
\texttt{task2} & 3 $\pm$ 4 & 17 $\pm$ 6 & \bt{42 $\pm$ 28} & 5 $\pm$ 6 & 0 $\pm$ 0 & 0 $\pm$ 0 & 1 $\pm$ 2 & 0 $\pm$ 0 & \hc{27 $\pm$ 20}\\
\texttt{task3} & 17 $\pm$ 9 & 76 $\pm$ 11 & 67 $\pm$ 14 & 0 $\pm$ 0 & 4 $\pm$ 4 & 12 $\pm$ 8 & 7 $\pm$ 5 & \bt{100 $\pm$ 0} & \bt{100 $\pm$ 0}\\
\texttt{task4} & 23 $\pm$ 10 & 79 $\pm$ 8 & 13 $\pm$ 18 & 1 $\pm$ 2 & 56 $\pm$ 12 & 44 $\pm$ 27 & 41 $\pm$ 13 & 3 $\pm$ 5 & \bt{91 $\pm$ 9}\\
\texttt{task5} & 30 $\pm$ 22 & \bt{87 $\pm$ 5} & 21 $\pm$ 8 & 0 $\pm$ 0 & 49 $\pm$ 16 & 48 $\pm$ 16 & 64 $\pm$ 14 & 44 $\pm$ 39 & \hc{76 $\pm$ 12}\\
\textbf{\texttt{overall}} & 25 $\pm$ 3 & 71 $\pm$ 4 & 33 $\pm$ 7 & 7 $\pm$ 7 & 30 $\pm$ 5 & 34 $\pm$ 7 & 33 $\pm$ 5 & 48 $\pm$ 10 & \bt{79 $\pm$ 8}\\
\midrule
\multicolumn{10}{l}{\textbf{\texttt{antmaze-large-navigate-oraclerep-v0}}} \\
\texttt{task1} & 3 $\pm$ 3 & 13 $\pm$ 11 & \bt{87 $\pm$ 9} & 51 $\pm$ 15 & 6 $\pm$ 5 & 6 $\pm$ 6 & 57 $\pm$ 13 & 11 $\pm$ 13 & \bt{83 $\pm$ 7}\\
\texttt{task2} & 22 $\pm$ 12 & 24 $\pm$ 8 & 64 $\pm$ 25 & \bt{69 $\pm$ 12} & 23 $\pm$ 6 & 26 $\pm$ 11 & \bt{66 $\pm$ 4} & 19 $\pm$ 5 & \hc{45 $\pm$ 15}\\
\texttt{task3} & 53 $\pm$ 10 & 48 $\pm$ 6 & 91 $\pm$ 3 & \bt{96 $\pm$ 3} & 69 $\pm$ 9 & 67 $\pm$ 4 & 80 $\pm$ 4 & 55 $\pm$ 6 & \hc{89 $\pm$ 3}\\
\texttt{task4} & 17 $\pm$ 9 & 7 $\pm$ 3 & \bt{92 $\pm$ 4} & 57 $\pm$ 16 & 17 $\pm$ 10 & 14 $\pm$ 1 & 8 $\pm$ 8 & 4 $\pm$ 3 & \hc{12 $\pm$ 4}\\
\texttt{task5} & 17 $\pm$ 5 & 16 $\pm$ 4 & \bt{90 $\pm$ 6} & 61 $\pm$ 12 & 17 $\pm$ 4 & 18 $\pm$ 5 & 18 $\pm$ 11 & 18 $\pm$ 12 & 20 $\pm$ 4\\
\textbf{\texttt{overall}} & 22 $\pm$ 4 & 22 $\pm$ 4 & \bt{85 $\pm$ 7} & 67 $\pm$ 8 & 27 $\pm$ 3 & 26 $\pm$ 3 & 46 $\pm$ 5 & 21 $\pm$ 6 & \hc{50 $\pm$ 7}\\
\midrule
\multicolumn{10}{l}{\textbf{\texttt{humanoidmaze-medium-navigate-oraclerep-v0}}} \\
\texttt{task1} & 10 $\pm$ 7 & 7 $\pm$ 9 & \bt{91 $\pm$ 8} & 8 $\pm$ 12 & 4 $\pm$ 3 & 6 $\pm$ 4 & 76 $\pm$ 6 & 29 $\pm$ 4 & 29 $\pm$ 7\\
\texttt{task2} & 6 $\pm$ 4 & 8 $\pm$ 5 & \bt{96 $\pm$ 1} & 22 $\pm$ 29 & 7 $\pm$ 3 & 9 $\pm$ 2 & \bt{96 $\pm$ 2} & 37 $\pm$ 12 & 32 $\pm$ 4\\
\texttt{task3} & 8 $\pm$ 5 & 15 $\pm$ 5 & \bt{72 $\pm$ 17} & 30 $\pm$ 21 & 12 $\pm$ 3 & 13 $\pm$ 3 & 8 $\pm$ 8 & 16 $\pm$ 7 & \hc{28 $\pm$ 6}\\
\texttt{task4} & 2 $\pm$ 1 & 3 $\pm$ 3 & 8 $\pm$ 7 & 9 $\pm$ 9 & 0 $\pm$ 0 & 1 $\pm$ 1 & \bt{11 $\pm$ 3} & 0 $\pm$ 0 & 2 $\pm$ 3\\
\texttt{task5} & 11 $\pm$ 5 & 12 $\pm$ 2 & \bt{93 $\pm$ 4} & 20 $\pm$ 19 & 14 $\pm$ 4 & 14 $\pm$ 7 & \bt{92 $\pm$ 3} & 39 $\pm$ 16 & 40 $\pm$ 10\\
\textbf{\texttt{overall}} & 7 $\pm$ 2 & 9 $\pm$ 3 & \bt{72 $\pm$ 5} & 18 $\pm$ 17 & 7 $\pm$ 0 & 9 $\pm$ 2 & 57 $\pm$ 1 & 24 $\pm$ 4 & 26 $\pm$ 6\\
\midrule
\multicolumn{10}{l}{\textbf{\texttt{cube-single-play-oraclerep-v0}}} \\
\texttt{task1} & 7 $\pm$ 9 & 17 $\pm$ 5 & 64 $\pm$ 6 & 6 $\pm$ 7 & 4 $\pm$ 3 & 7 $\pm$ 7 & \bt{98 $\pm$ 2} & 87 $\pm$ 3 & \bt{95 $\pm$ 2}\\
\texttt{task2} & 8 $\pm$ 6 & 18 $\pm$ 12 & 61 $\pm$ 8 & 8 $\pm$ 4 & 5 $\pm$ 4 & 6 $\pm$ 12 & \bt{97 $\pm$ 3} & 93 $\pm$ 4 & \bt{100 $\pm$ 0}\\
\texttt{task3} & 9 $\pm$ 6 & 22 $\pm$ 5 & 69 $\pm$ 9 & 9 $\pm$ 3 & 1 $\pm$ 2 & 2 $\pm$ 3 & \bt{99 $\pm$ 1} & 93 $\pm$ 2 & \bt{94 $\pm$ 3}\\
\texttt{task4} & 6 $\pm$ 2 & 20 $\pm$ 6 & 56 $\pm$ 10 & 2 $\pm$ 2 & 3 $\pm$ 2 & 27 $\pm$ 17 & 93 $\pm$ 6 & 85 $\pm$ 7 & \bt{98 $\pm$ 2}\\
\texttt{task5} & 5 $\pm$ 5 & 14 $\pm$ 7 & 66 $\pm$ 18 & 3 $\pm$ 2 & 0 $\pm$ 0 & 24 $\pm$ 21 & \bt{87 $\pm$ 7} & 82 $\pm$ 4 & \bt{90 $\pm$ 4}\\
\textbf{\texttt{overall}} & 7 $\pm$ 2 & 18 $\pm$ 5 & 63 $\pm$ 5 & 6 $\pm$ 3 & 3 $\pm$ 1 & 13 $\pm$ 8 & \bt{95 $\pm$ 2} & 88 $\pm$ 2 & \bt{95 $\pm$ 2}\\
\midrule
\multicolumn{10}{l}{\textbf{\texttt{cube-double-play-oraclerep-v0}}} \\
\texttt{task1} & 6 $\pm$ 5 & 17 $\pm$ 8 & 77 $\pm$ 8 & 7 $\pm$ 4 & 6 $\pm$ 2 & 1 $\pm$ 2 & 73 $\pm$ 5 & 88 $\pm$ 4 & \bt{96 $\pm$ 3}\\
\texttt{task2} & 0 $\pm$ 0 & 1 $\pm$ 1 & 42 $\pm$ 9 & 0 $\pm$ 0 & 0 $\pm$ 0 & 0 $\pm$ 0 & 23 $\pm$ 7 & 78 $\pm$ 5 & \bt{90 $\pm$ 3}\\
\texttt{task3} & 0 $\pm$ 0 & 0 $\pm$ 0 & 39 $\pm$ 10 & 0 $\pm$ 0 & 0 $\pm$ 0 & 0 $\pm$ 0 & 30 $\pm$ 11 & 75 $\pm$ 6 & \bt{93 $\pm$ 1}\\
\texttt{task4} & 0 $\pm$ 0 & 1 $\pm$ 2 & 1 $\pm$ 1 & 0 $\pm$ 0 & 0 $\pm$ 0 & 0 $\pm$ 0 & 3 $\pm$ 3 & 8 $\pm$ 5 & \bt{16 $\pm$ 5}\\
\texttt{task5} & 0 $\pm$ 0 & 2 $\pm$ 2 & 17 $\pm$ 4 & 0 $\pm$ 0 & 0 $\pm$ 0 & 0 $\pm$ 0 & 18 $\pm$ 7 & 47 $\pm$ 14 & \bt{56 $\pm$ 4}\\
\textbf{\texttt{overall}} & 1 $\pm$ 1 & 4 $\pm$ 2 & 35 $\pm$ 3 & 1 $\pm$ 1 & 1 $\pm$ 0 & 0 $\pm$ 0 & 30 $\pm$ 5 & 59 $\pm$ 2 & \bt{70 $\pm$ 3}\\
\midrule
\multicolumn{10}{l}{\textbf{\texttt{scene-play-oraclerep-v0}}} \\
\texttt{task1} & 14 $\pm$ 8 & 60 $\pm$ 6 & 71 $\pm$ 13 & 19 $\pm$ 8 & 43 $\pm$ 4 & 32 $\pm$ 8 & \bt{97 $\pm$ 2} & \bt{97 $\pm$ 3} & \bt{98 $\pm$ 2}\\
\texttt{task2} & 2 $\pm$ 1 & 13 $\pm$ 3 & 14 $\pm$ 4 & 2 $\pm$ 2 & 9 $\pm$ 2 & 2 $\pm$ 2 & \bt{95 $\pm$ 3} & \bt{92 $\pm$ 4} & \bt{90 $\pm$ 4}\\
\texttt{task3} & 2 $\pm$ 3 & 21 $\pm$ 9 & 33 $\pm$ 9 & 1 $\pm$ 1 & 3 $\pm$ 4 & 2 $\pm$ 3 & \bt{97 $\pm$ 3} & 83 $\pm$ 8 & 79 $\pm$ 3\\
\texttt{task4} & 3 $\pm$ 1 & 19 $\pm$ 8 & 12 $\pm$ 3 & 6 $\pm$ 3 & 2 $\pm$ 2 & 3 $\pm$ 1 & 76 $\pm$ 17 & 43 $\pm$ 28 & \bt{86 $\pm$ 4}\\
\texttt{task5} & 0 $\pm$ 0 & 3 $\pm$ 4 & 2 $\pm$ 2 & 1 $\pm$ 1 & 0 $\pm$ 0 & 0 $\pm$ 0 & 18 $\pm$ 9 & 23 $\pm$ 16 & \bt{32 $\pm$ 4}\\
\textbf{\texttt{overall}} & 4 $\pm$ 2 & 23 $\pm$ 2 & 26 $\pm$ 5 & 6 $\pm$ 2 & 12 $\pm$ 1 & 8 $\pm$ 2 & \bt{77 $\pm$ 2} & 68 $\pm$ 10 & \bt{77 $\pm$ 3}\\
\midrule
\multicolumn{10}{l}{\textbf{\texttt{puzzle-3x3-play-oraclerep-v0}}} \\
\texttt{task1} & 4 $\pm$ 4 & 10 $\pm$ 1 & 15 $\pm$ 8 & 3 $\pm$ 3 & 5 $\pm$ 1 & 7 $\pm$ 3 & \bt{99 $\pm$ 1} & 4 $\pm$ 2 & \hc{48 $\pm$ 3}\\
\texttt{task2} & 1 $\pm$ 2 & 1 $\pm$ 1 & 6 $\pm$ 3 & 0 $\pm$ 0 & 1 $\pm$ 1 & 2 $\pm$ 2 & \bt{99 $\pm$ 1} & 2 $\pm$ 2 & \hc{10 $\pm$ 3}\\
\texttt{task3} & 1 $\pm$ 1 & 1 $\pm$ 1 & 1 $\pm$ 1 & 0 $\pm$ 0 & 0 $\pm$ 0 & 1 $\pm$ 2 & \bt{100 $\pm$ 0} & 2 $\pm$ 1 & 2 $\pm$ 1\\
\texttt{task4} & 0 $\pm$ 0 & 1 $\pm$ 1 & 1 $\pm$ 1 & 0 $\pm$ 0 & 2 $\pm$ 3 & 1 $\pm$ 1 & \bt{98 $\pm$ 1} & 1 $\pm$ 1 & \hc{10 $\pm$ 2}\\
\texttt{task5} & 1 $\pm$ 1 & 2 $\pm$ 2 & 2 $\pm$ 2 & 0 $\pm$ 0 & 1 $\pm$ 1 & 0 $\pm$ 0 & \bt{99 $\pm$ 1} & 1 $\pm$ 1 & 4 $\pm$ 2\\
\textbf{\texttt{overall}} & 1 $\pm$ 1 & 3 $\pm$ 1 & 5 $\pm$ 1 & 1 $\pm$ 1 & 2 $\pm$ 0 & 2 $\pm$ 0 & \bt{99 $\pm$ 0} & 2 $\pm$ 1 & \hc{15 $\pm$ 2}\\
\midrule
\multicolumn{10}{l}{\textbf{\texttt{puzzle-4x4-play-oraclerep-v0}}} \\
\texttt{task1} & 1 $\pm$ 1 & 1 $\pm$ 1 & 1 $\pm$ 1 & 0 $\pm$ 0 & 1 $\pm$ 1 & 0 $\pm$ 0 & 47 $\pm$ 5 & 6 $\pm$ 4 & \bt{62 $\pm$ 5}\\
\texttt{task2} & 0 $\pm$ 0 & 1 $\pm$ 1 & 0 $\pm$ 0 & 0 $\pm$ 0 & 0 $\pm$ 0 & 1 $\pm$ 1 & \bt{17 $\pm$ 5} & 4 $\pm$ 4 & \hc{12 $\pm$ 4}\\
\texttt{task3} & 0 $\pm$ 0 & 1 $\pm$ 1 & 1 $\pm$ 1 & 0 $\pm$ 0 & 1 $\pm$ 1 & 0 $\pm$ 0 & 38 $\pm$ 13 & 6 $\pm$ 1 & \bt{46 $\pm$ 7}\\
\texttt{task4} & 1 $\pm$ 1 & 1 $\pm$ 1 & 1 $\pm$ 1 & 0 $\pm$ 0 & 0 $\pm$ 0 & 0 $\pm$ 0 & \bt{34 $\pm$ 2} & 6 $\pm$ 6 & \bt{34 $\pm$ 3}\\
\texttt{task5} & 0 $\pm$ 0 & 0 $\pm$ 0 & 0 $\pm$ 0 & 0 $\pm$ 0 & 1 $\pm$ 1 & 0 $\pm$ 0 & 32 $\pm$ 6 & 6 $\pm$ 1 & \bt{43 $\pm$ 10}\\
\textbf{\texttt{overall}} & 0 $\pm$ 0 & 1 $\pm$ 0 & 0 $\pm$ 0 & 0 $\pm$ 0 & 0 $\pm$ 0 & 0 $\pm$ 0 & 34 $\pm$ 4 & 5 $\pm$ 2 & \bt{39 $\pm$ 6}\\
\bottomrule
\end{tabular}%
}
\end{table*}

\end{document}